%% file: neurips_2026.tex
\documentclass{article}

 \usepackage[preprint]{neurips_2026}

\usepackage[utf8]{inputenc} 
\usepackage{hyperref}       

\usepackage{url}            
\usepackage{booktabs}       
\usepackage{amsfonts}       
\usepackage{nicefrac}       
\usepackage{microtype}      
\usepackage{xcolor}         
\usepackage{graphicx}
\usepackage{tikz}
\usepackage{tikz-cd}
\usepackage{dsfont}
\usepackage{amsfonts}
\usepackage{amsmath}
\usepackage{amsthm}
\usepackage{bbm}
\usepackage{cancel}
\usepackage{comment}
\usepackage[linesnumbered,ruled,vlined]{algorithm2e}
\usepackage{algpseudocode}
\usepackage{amsmath}
\usepackage{stmaryrd}
\usepackage{mleftright}
\usepackage{xcolor}
\usepackage{amssymb}

\usepackage[inkscapeversion=1, inkscapelatex=false]{svg}

\usepackage{apptools}
\AtAppendix{\counterwithin{lemma}{section}}
\AtAppendix{\counterwithin{theorem}{section}}
\AtAppendix{\counterwithin{proposition}{section}}

\input{math_commands}

\definecolor{Red}{HTML}{E74C3C}  
\definecolor{Blue}{HTML}{2E86C1}  
\definecolor{Green}{HTML}{27AE60}  
\definecolor{Orange}{HTML}{D35400}
\definecolor{Purple}{HTML}{8E44AD} \definecolor{Yellow}{HTML}{F39C12}
\newcommand{\IA}{\textcolor{Red}{A}}
\newcommand{\IB}{\textcolor{Blue}{B}}
\newcommand{\IC}{\textcolor{Green}{C}}

\newcommand{\cvc}[1]{{\color{orange} #1}}
\newcommand{\sq}[1]{{\color{Blue} #1}}
\newcommand{\mc}[1]{{\color{teal} #1}}

\newtheorem{remark}{Remark}
\newtheorem{Lemma}{Lemma}

\title{Bridging Maximum Likelihood and Optimal Transport for Efficient Inference and Model Selection in Stochastic Block Models}

\author{%
  Simon Queric \\
  Université Côte d’Azur, Inria, CNRS\\
  LJAD, Maasai\\
  Nice, France \\
  \texttt{simon.queric@inria.fr} \\
  \And
  Cédric Vincent-Cuaz\\
  EPFL \\
  Lausanne, Switzerland \\
  \And
  Charles Bouveyron \\
  Université Côte d’Azur, Inria, CNRS\\
  LJAD, Maasai\\
  Nice, France\\
  \And
  Marco Corneli \\
  Université Côte d’Azur, Inria, CNRS\\
  CEPAM, Maasai\\
  Nice, France\\
}

\begin{document}
\maketitle
\begin{abstract}
We study inference in stochastic block models (SBMs) through the lens of optimal transport (OT). We first establish that maximum likelihood variational inference (MLVI) can be interpreted as a semi-relaxed Gromov–Wasserstein (srGW) projection with entropic regularization. While this formulation yields accurate clustering, the entropic regularization prevents  transport plans to be sparse,  hindering intrinsic model selection. Consequently, we investigate unregularized srGW estimators,
and prove that they consistently recover both the SBM connectivity matrix and latent cluster assignments in the asymptotic regime. However, this asymptotic property does not translate into reliable model selection in finite samples, and calls for additional mechanisms to promote sparsity in the inferred cluster proportions.
We empirically show that such a regularized formulation yields estimators that simultaneously recover model parameters and select the number of clusters in a single optimization problem, thereby avoiding costly grid search or heuristic model selection procedures.
\end{abstract}

\section{Introduction}

Although some recent works investigate connections between OT and either mixture models~\citep{vayer2025noterelationsmixturemodels,SlicedGMM} or maximum likelihood estimation~\citep{rigollet2018EOTML},  
until now, very little has been explored concerning links between Gromov-Wasserstein OT and stochastic block modeling.
Let us first (briefly) introduce each of these two topics. Stochastic block models~\citep[SBMs,][]{holland1983stochastic,nowicki2001estimation,Daudin} are a popular family of generative models for random graphs, widely used in network analysis~\citep{borgatti2009network}. 
In its simplest formulations, SBM assumes that the nodes of a graph are partitioned in hidden groups (or \emph{clusters}) and the probability of a tie between two nodes only depends on the clusters they belong to. In the context of a single observed graph (the one we care about in this work), fitting SBM to the graph can be done to perform nodes clustering, as an alternative to a huge number of other approaches including modularity maximization~\citep[][]{Louvain,fortunato2010community}, spectral clustering~\citep[][]{von2007tutorial,Lei_2015}, graph neural networks~\citep[][]{kawamoto2018mean,shah2024neurocut}, optimal transport~\citep{xu2019scalable,srGW} and kernel methods~\citep{kloster2014heat}, just to cite a few. Although being far from capturing some stylized features of real networks~\citep[see for instance][]{karrer2011stochastic}, SBMs have at least two key advantages with respect to a number of the alternative approaches cited above: i) they can recover clusters of nodes other than communities (e.g. hubs, non assortative structures, etc.) and  ii) they come with Bayesian model selection tools, mostly Integrated Classification Likelihood~\citep[ICL,][]{biernacki2000assessing}, allowing one to estimate the number of clusters of nodes, whose prior knowledge is instead required by many other methods.



\paragraph{Inference in SBMs and statistical guarantees} 
Inference in SBMs has been carried out in many ways including purely Bayesian solutions, such as Gibbs~\citep{nowicki2001estimation} or allocation~\citep{mcdaid2013improved} sampling, as well as greedy search schemes based on the exact ICL~\citep{Greed21}, allowing one to cluster the graph's nodes and select the number of clusters together, in a computationally efficient way.
The first results proving asymptotic consistency of maximum likelihood and variational estimators appeared in the seminal paper~\cite{Daudin}. A series of works then studied minimax optimality for least squares estimations \cite{Gao_2015} and, more recently, maximum likelihood ~\citep{GAUCHERMLE} and variational estimators~\citep{GaucherSBM} were derived in the general and challenging case of missing observations. 
Others lines of work aim at inspecting to what extent the clusters of nodes in graphs generated by SBMs can be retrieved by techniques such as spectral clustering~\citep{saade2014spectralclusteringgraphsbethe,Lei_2015} 
or agnostic-sphere-comparison~\cite{abbe2015}.
The above cited works mainly show asymptotic results.

The second topic of interest in this paper is Optimal Transport (OT) for graph data. OT~\cite{OTVillani,COT,pereira2025survey}  provides a principled framework for comparing probability distributions with broad applications in statistical learning. In its classical formulation, OT operates on distributions defined on a \emph{common ambient space}, which limits its applicability to structured objects, such as graphs defined by pairwise relationships, that may not share a common representation space. The recent Gromov-Wasserstein (GW) framework overcomes this limitation by comparing objects through their \emph{internal geometry}. Rather than aligning points directly, GW aligns pairwise relationships, enabling the comparison of distributions across different metric spaces~\cite{GWMemoli,sturm2023space}, and more generally of graphs represented by relational matrices together with probability distributions over their nodes \cite{chowdhury2019gromov}. In this setting, GW defines a permutation-invariant metric between graphs, making it a natural tool for graph matching \cite{xu2019gromov,xu2019scalable,chowdhury2021quantized} and kernel-based classification \cite{PeyreGW,OTstructureddata}.
Beyond pairwise comparison, a key strength of GW lies in its ability to \emph{learn latent structures}. Instead of comparing two fixed objects, one can seek a simplified representation that best approximates an observed graph under a GW-based discrepancy. This leads to a family of formulations in which the target structure, the associated weights, or both are learned jointly with the coupling. Depending on which components are fixed or optimized, this general principle gives rise to a wide range of applications on graphs, including partitioning \cite{xu2019scalable,chowdhury2021generalized,srGW,OTclusteringattributed}, graphon estimation \cite{xu2021learning} and representation learning (e.g., graph coarsening\cite{chen2023gromov}, dictionary learning \cite{xu2020gromov,vincent2021online,liu2022robust,zeng2023generative}, graph neural networks \cite{vincent2022template,chu2023wasserstein,qian2024reimagining,krzakala2026the}, or dimensionality reduction \cite{van2025distributional,clark2025generalized}). In particular, learning a factored structure together with node assignments closely parallels the objectives of SBMs, where one aims to infer both cluster memberships and connectivity patterns.

\paragraph{Contributions} In the present work, we inspect some interesting links between semi-relaxed Gromov Wasserstein~\citep{srGW}, a divergence derived from the GW distance and the stochastic block model. More precisely, we show that one can infer the parameters and cluster assignments of SBMs leveraging srGW: a new family of estimators for SBMs is introduced together with a method to perform automatic model selection (i.e. selecting the number of groups), based on a sparsity promoting regularization of the OT cost. 
Notably, this allows one to perform clustering and select the number of clusters in one shot, thus avoiding computationally expensive grid searches or greedy methods with not guarantee of convergence. 
We draw connections between standard variational expectation maximization in SBMs and OT-based estimation and prove the asymptotic consistency of our estimators. An algorithm to compute OT estimates and perform graph partitioning is provided. We perform numerical experiments, illustrating the interest of the approach and providing new insights about model selection for latent block models.

The paper is organized as follows: we first revise Gromov-Wasserstein OT, in Section~\ref{sec:GW}, and stochastic block modelling, in  Section~\ref{sec:SBM}. This order is chosen for pedagogical reasons as it will be clear in the following. In Section~\ref{sec:main} we present our main results. We introduce the OT estimators and state some of their properties, before detailing an algorithm to compute them, equipped with a model selection routine.   
Finally, Section~\ref{sec:exp} illustrates with numerical experiments on simulated data the main appeals and limitations of our estimators.
\paragraph{Notation} In the following, $\mathbf{1}_N$ denotes a column vector of $N$ ones. We denote probability simplex with $N$ bins by $\Delta_N := \{\bh \in \R_+^N | \bh^\top \mathbf{1}_N = 1\}$. Given a vector $\mathbf{v} \in \mathbb{R}^N_+$ we denote by $\mathcal{U}_K(\mathbf{v})$ the following set $\lbrace \bT \in \mathbb{R}^{N \times K}_+ | \bT \mathbf{1}_N = \mathbf{v}\rbrace$.
The shorthand notation $\prod_{j>i} \prod_{k,l} := \prod_{i=1}^N\prod_{j>i}^N \prod_{k=1}^K\prod_{l=1}^K$ is adopted. Similarly for $\sum_{j>i}\sum_{k,l}$.
Unless differently stated, it is assumed that graphs are undirected with no self loops.

\section{Gromov-Wasserstein}\label{sec:GW}

\paragraph{Gromov-Wasserstein Discrepancy for graphs} Let $\bA \in \mathbb{R}^{N \times N}$ be the adjacency matrix, possibly weighted, of a graph with $N$ nodes, endowed with a probability distribution $\bh \in \Delta_N$. In practice, $\bh$ is often taken to be uniform ($\bh = \frac{1}{N}\mathbf{1}_N$), or derived from structural properties of the graph, such as normalized node degrees \cite{xu2019scalable}. Similarly, let $\bTheta \in \mathbb{R}^{K \times K}$ denote a target structure (e.g., a connectivity matrix) with associated weights $\balpha \in \Delta_K$. The Gromov-Wasserstein (GW) discrepancy between the source graph $(\bA, \bh)$ and the target graph $(\bTheta, \balpha)$ seeks a transport matrix, i.e a soft correspondence between nodes, that best preserves pairwise relationships. It is defined as 
\begin{equation}\label{eq:GW_def}
\GW_{\ell}(\bA,\bh, \bTheta, \balpha) := 
\min_{\bT \in \Pi(\bh, \balpha)} \sum_{i,j,k,l} \ell(A_{ij}, \Theta_{kl}) \, T_{ik} T_{jl},
\end{equation}
where $\Pi(\bh, \balpha) = \{ \bT \in \mathbb{R}_+^{N \times K}|~\bT \mathbf{1}_K = \bh, ~\bT^\top \mathbf{1}_N = \balpha \}$
denotes the set of admissible couplings with prescribed source and target marginals. The function $\ell$ measures the discrepancy between edge weights. A common choice is the quadratic loss $\ell_2(a,b) = (a-b)^2$, which leads to the classical GW formulation. More general losses, which can be decomposed as
$\ell(a,b) = f_1(a) + f_2(b) - h_1(a) h_2(b)$ \cite{PeyreGW}, like the Kullback-Leibler or binary cross-entropy losses, have been empirically studied \cite{van2025distributional}. These composite losses lead to quadratic programs in $\bT$, which are in general non-convex and computationally challenging. Nevertheless, GW with any $\ell_p$ loss enjoys strong theoretical properties: for arbitrary adjacency matrices, GW defines a metric over the space of \emph{weakly isomorphic graphs} \cite{chowdhury2019gromov}. This notion extends classical graph isomorphism (i.e., invariance under node permutations) by allowing mass to be split across nodes. In particular, two graphs are equivalent in the GW sense if one can be obtained from the other by duplicating nodes and redistributing their associated mass while preserving connectivity patterns. Notably, these properties motivated the application of GW for graph partitioning by matching an input graph, with connectivity matrix $\bA$ or an associated heat kernel,  to an ideal target graph $(\bI_K, \balpha)$, corresponding to perfectly disconnected clusters with proportions $\balpha$, derived for instance from power laws over normalized degrees of $\bA$ \cite{xu2019scalable, chowdhury2021generalized}.

\paragraph{GW minimal estimators}\label{GWminestimator} 
To improve graph partitioning via a matching to an idealized graph $(\bI_K, \balpha)$, \cite{srGW} proposed a form of GW minimal estimator that consists in optimizing the cluster proportions estimated by $\balpha$ in the GW problem \ref{eq:GW_def}. To this end, authors observed that jointly optimizing the GW loss over $\balpha \in \Sigma_K$ and  $\bT \in \mathcal{U}(\bh, \balpha)$ can be replaced by a single simpler optimization over $\bT \in \mathcal{U}_K(\bh):=\{\bT \in \mathbb{R}_+^{N \times K} |~\bT \mathbf{1}_K = \bh\}$, leading to the definition of the semi-relaxed Gromov-Wasserstein divergence \cite{srGW}:
\begin{equation}
\srGW_{\ell}(\bA, \bh, \bTheta) := \min_{\bT \in \mathcal{U}_K(\bh)} \sum_{i,j,k,l} \ell(A_{ij}, \Theta_{kl}) T_{ik} T_{jl} \quad = \min_{\balpha \in \Delta_K} \GW_{\ell}(\bA,\bh, \bTheta, \balpha). 
\label{eq:srGW_def}
\end{equation} 
A natural extension to better factor the input graph consists in additionally learning the target structure, solving for
\begin{equation}\label{eq:GW_barycenter}
\min_{\bTheta \in \R^{K \times K}} \srGW_\ell(\bA, \bh, \bTheta) \Longleftrightarrow \min_{\bTheta \in \R^{K\times K}, \balpha \in \Delta_K} \GW_\ell(\bA, \bh, \bTheta, \balpha).
\end{equation}
This comprehensive GW minimal estimator, also known as srGW barycenter in the literature, has relevant theoretical guarantees for graph partitioning when $\ell=\ell_2$. Specifically,
\citet{van2025distributional} established that, under conditional positive or negative definiteness of the input structure $\bA$, there exists (hard-clustering) membership matrices which are solutions. Moreover, when $\bA$ is positive definite, the problem is equivalent to the weighted kernel $k$-means \cite{dhillon2004kernel}, which is well-suited for graph coarsening and partitioning due to its spectral preservation properties \cite{chen2023gromov}. However, such guarantees do not extend to general adjacency matrices, including those generated from random graph models like SBMs.

\section{Stochastic Block Models}\label{sec:SBM}
Given a graph with $N$ nodes and its binary adjacency matrix $\bA \in \{0,1\}^{N\times N}$,
in the Stochastic Block Model~\citep[SBM,][]{Daudin} it is assumed that the nodes are split into
$K$ clusters with proportions $\balpha := (\alpha_1, \dots, \alpha_K)^T$, where $\balpha \in \Delta_K$ and $\alpha_k$ denotes the probability that a node belongs to the $k$-th cluster. Furthermore, $K$ doesn't depend on $N$.
More precisely, a latent vector $Z:=(Z_1,\dots, Z_N)^T$ is introduced such that
\begin{equation}
  Z_i \overset{\text{i.i.d}}{\sim} \text{Cat}(\balpha)
\end{equation}
denotes the cluster of the i-th node and $\mathbb{P}(Z_{ik} = 1) = \alpha_k$\footnote{As it is common in the literature we interchangeably denote by $Z_i$ a random number in $\llbracket K \rrbracket$ or a binary random vector in $\{ 0,1\}^K$. By the way, both $Z_i = k$ and $Z_{ik}=1$ mean that the $i$-th node is in the $k$-th cluster.}. The main assumption in SBM is
\begin{equation}
A_{ij} | (Z_i, Z_j) = (k,l) \overset{\text{ind.}}{\sim} \mathcal{B}(\Theta_{kl}),     
\label{eq:SBM_main}
\end{equation}
where $\mathcal{B}(\cdot)$ denotes the Bernoulli probability distribution and $\bTheta \in [0,1]^{K\times K}$ is the  connectivity matrix. As it can be seen, the probability that two nodes connect with each other only depends on the clusters they are in. Several extensions of SBM have been introduced to deal with valued graphs~\citep[see][and references therein]{lee2019review}. In those cases the entries of $\bA$ can be integer, categorical or real values and the Bernoulli distribution in Eq.~\eqref{eq:SBM_main} is replaced by an appropriate probability density or mass function $p(\cdot |\Theta_{kl})$, possibly zero inflated, i.e. $p(\cdot, \Theta_{kl}) = \rho \mu(\cdot|\Theta_{kl}) + (1-\rho)\delta_{\{0\}}(\cdot)$, with $\rho \in [0,1]$ and $\mu$ denoting another probability density or mass function. Figure~\ref{fig1:Bsbm} displays an example of a graph generated by a Bernoulli SBM with five communities and the connectivity matrix of the corresponding model. The likelihood of the adjacency matrix $\bA$ under $\text{SBM}(\bTheta, \balpha)$ reads
\begin{equation}
    p(\bA|\bTheta, \balpha) = \sum_{Z \in \mathcal{Z}_{N,K}}\left(\prod_{j>i}\prod_{k,l}p(A_{ij}| \Theta_{kl})^{Z_{ik}Z_{jl}} \prod_i \prod_k \alpha_k^{Z_{ik}}\right)
    \label{eq:SBM_likeli}
\end{equation}
where the first sum on the r.h.s of the equality is taken over all possible assignments of $N$ nodes in $K$ clusters.

\paragraph{Identifiability} 
As for other latent variable models, the parameters of SBM are identifiable up to label switching. That is, if $P_\sigma$ is any permutation matrix, then it can be shown that $\text{SBM}(\bTheta, \balpha)$ and $\text{SBM}(P_\sigma \bTheta P_\sigma^T, P_\sigma \balpha)$ yeld to the same likelihood in Eq.~\eqref{eq:SBM_likeli} for any $\bA$, binary or not. 
Indeed, the identifiability of the SBM parameters is slightly more subtle and label switching is not the only issue one has to face to enforce identifiability. The reader is referred to~\cite{ConsistencySBM,mariadassou2015convergence} for an in depth treatment of this topic. However, a general result proved in~\cite{ConsistencySBM} is that the vector $\bTheta\balpha \in \mathbb{R}^K$ should have distinct coordinates in order for the model parameters to be identifiable. That result is intimately connected with the following assumptions that we need in order to state our results.

\textbf{Assumptions.}\label{Ass:ass}
We assume that
    \begin{enumerate}
      \item[\textbf{A1.}] $\forall k\neq k' \quad \exists l \in \llbracket K \rrbracket$ such that $\Theta_{kl} \neq \Theta_{k'l}$ or $\Theta_{l k} \neq \Theta_{l k'}$
        \item[\textbf{A2.}] $\exists \zeta \in ]0, 1[ \quad \forall k,l \in \llbracket K \rrbracket^2, \ \Theta_{kl} \in [\zeta, 1-\zeta]$
        \item[$\textbf{A3.}$] $\exists \gamma \in ]0, 1/K[ $ such that $\alpha_k \in [\gamma, 1-\gamma]$
    \end{enumerate}
The first assumption states there are no two identical rows/columns in $\bTheta$. The second assumption is technical, allowing us to bound $\log \Theta_{kl}$ in the binary case (i.e. Bernoulli SBM) and the last assumption forbids empty clusters. The numbers $\zeta$ and $\gamma$ do not depend on the number of nodes $N$.



\begin{figure}[t!]
    \centering
    \includegraphics[width=0.4\linewidth]{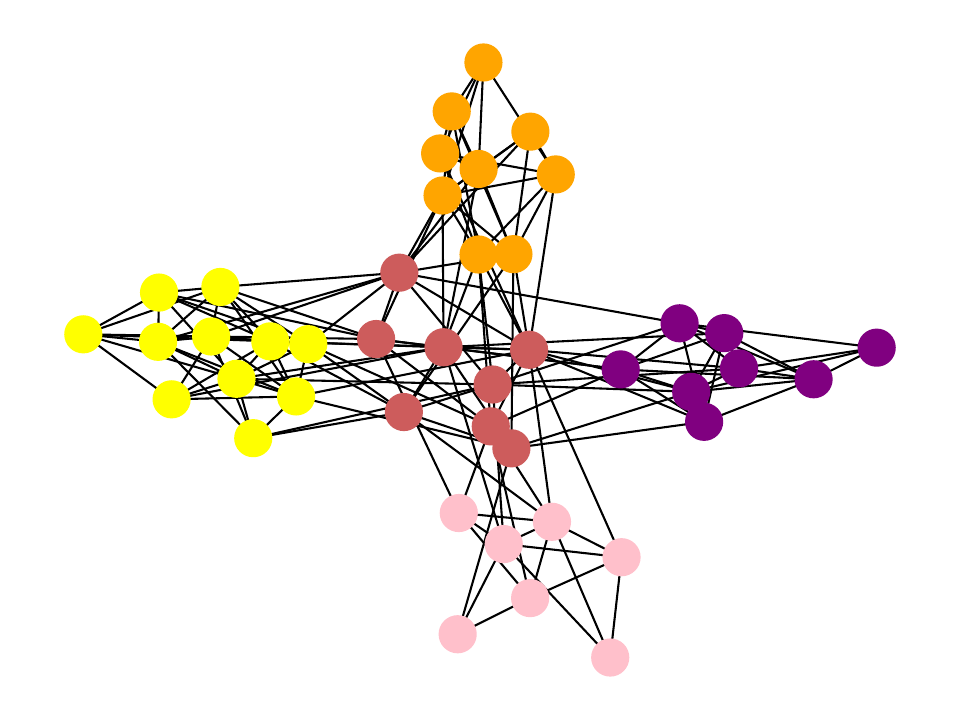}
     \includegraphics[width=0.4\linewidth]{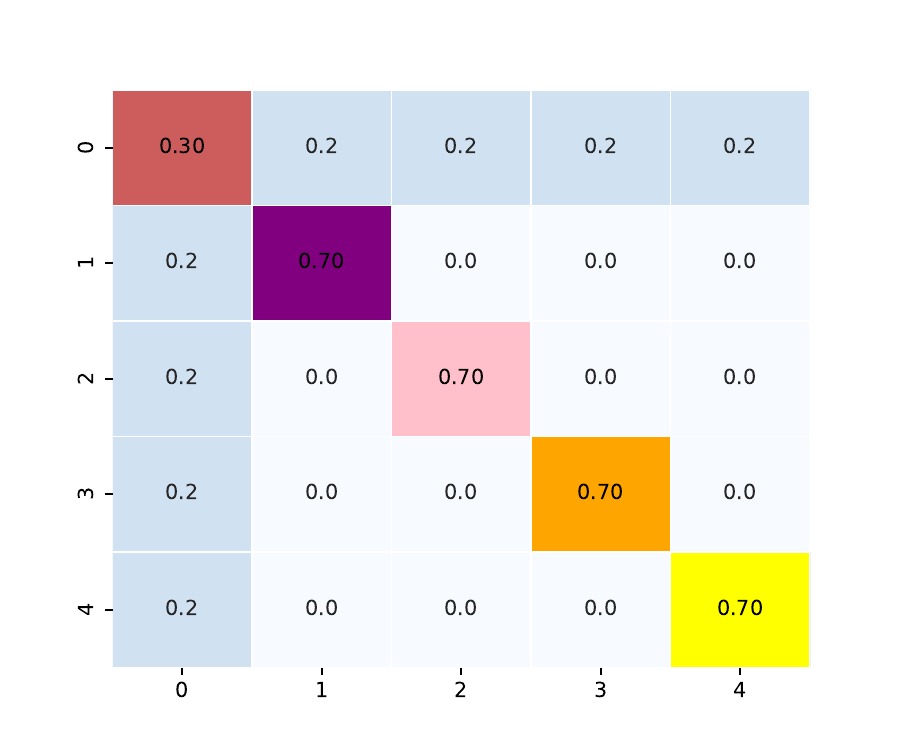}
    \caption{Left : A graph sampled from a Bernoulli SBM with five communities. Right : The connectivity matrix of the associated model.}
    \label{fig1:Bsbm}
\end{figure}

\subsection{Variational Inference}\label{subsec:SBM_VI}

In order to estimate the model parameters $(\bTheta,\balpha)$, one might seek to directly maximize the log-likelihood of the observed data $\log p(\bA|\bTheta, \balpha)$. However, the marginalization in Eq.~\eqref{eq:SBM_likeli} involves summing over $K^N$ possible label assignments. Because of the combinatorial nature of the problem, that quantity becomes very soon not tractable (neither its logarithm). Although several alternative inference strategies have been introduced in the literature, here we focus on variational Expectation Maximization (EM), as described in~\cite{Daudin}. 
First, the Jensen's inequality can be adopted in order to lower bound the log-likelihood of the observed data
\begin{equation}
\log p(\bA | \bTheta, \balpha) = \log \E_{Z \sim q}\left[{\frac{p(\bA, Z | \bTheta, \balpha)}{q(Z)}}\right] \geqslant \E_{Z \sim q}\left[{\log{\frac{p(\bA, Z | \bTheta, \balpha)}{q(Z)}}}\right]:=\mathcal{L}(q, \bTheta, \balpha),
\end{equation}
where $q(\cdot)$ is any distribution over $Z$ (with a proper support) and the last term on the right hand side is known as the Evidence Lower Bound (ELBO). The above inequality turns into in an equality iff $q(\cdot)$ is the posterior distribution of $Z$ given $\bA$ and the model parameters, however that quantity is not tractable in SBM~\citep{Daudin}. A common approach consists into adopting a mean-field approximation 
\begin{equation}
    q(Z) := \prod_{i=1}^N \prod_{k=1}^K{\tau_{ik}^{Z_{ik}}},
    \label{eq:variational}
\end{equation}
where the probability $\tau_{ik}$ is the $(i,k)$-th  entry of $\btau \in \mathcal{U}_K(\mathbf{1}_N)$, unknown and to be estimated. 
Then
the ELBO reads as follows
\begin{equation}
\begin{split}
 \mathcal{L}(q,\bTheta,\balpha) &= \frac{1}{2}\sum_{i\neq j}\sum_{k,l}\tau_{ik}\tau_{jl}\log p(A_{ij} | \Theta_{kl})
- \sum_{i,k}\tau_{ik}\log{\tau_{ik}} + \sum_{k} \left(\sum_i \tau_{ik}\right)\log{\alpha_{k}}, 
\end{split}
\label{eq:ELBO}
\end{equation}
where the factor 1/2 comes from the fact that our graph is undirected, with no self loops ($j \neq i$). 
The ELBO is typically maximized alternatively by (E Step) introducing Lagrange multipliers to account for the constraints on the variable $\boldsymbol{\tau}$ and look at first order conditions, leading to a fixed point algorithm; (M  Step) updating $(\bTheta, \balpha)$ using closed-form solutions provided by first order optimality conditions. 

\paragraph{ELBO as an entropic GW problem} 
It can easily be checked that the first order condition for $\balpha$ leads to the M update $\hat{\balpha}:=(N^{-1}\sum_{i}\tau_{ik})_{k \in \llbracket K \rrbracket}$.
By injecting that closed-form solution into Eq.~\eqref{eq:ELBO} and substituting $\btau = N\bT$, with $\bT \in \mathcal{U}_K(N^{-1}\mathbf{1}_N)$, the ELBO maximization is equivalent to the following minimization problem
\begin{equation}
\min_{\bT \in \mathcal{U}_K(\mathbf{1}_N/N), \bTheta } \left(-\sum_{i,j,k,l}  \log p(A_{ij}| \Theta_{kl}) T_{ik} T_{jl} - \frac{2}{N} [H(\bT) - H(\bT^\top \mathbf{1}_N) ]\right),
\label{eq:eq_ELBO_srGW}
\end{equation} 
where $H(\bT) = -\sum_{ik} T_{ik} \log T_{ik}$ denotes the entropy of $\bT$, $H(\bT^\top \mathbf{1}_N)$ the one of its optimized marginal and, since the diagonal of $\bA$ only contains zero entries, we assume by convention that $\forall i,k,l, \log p(A_{ii} | \Theta_{kl})= 0$.
When looking at the fist term inside parentheses, one clearly recognizes the objective of the srGW problem in~\eqref{eq:srGW_def}, with inner loss $\ell(\cdot,\cdot) = - \log p(\cdot | \cdot)$. In the above problem, however, an additional entropic regularization term appears, having important effects on the solutions for a fixed $N$. In the absence of regularization, GW problems admit sparse solutions due to their constraints on polytopes. In contrast, the standard practice in OT of maximizing $H(\bT)$ promotes diffuse solutions that improve robustness to noise and optimization smoothness \cite{COT,PeyreGW}. Conversely, minimizing $H(\bT^\top \mathbf{1}_N)$ encourages estimated cluster proportions to converge towards an extremity of the probability simplex \cite{van2024graph}. However, both entropy terms act as barrier functions \cite{benamou2015iterative}: they prevent the entries of $\bT$ from vanishing, and therefore yield dense solutions. As a consequence, they hinder the ability of the model to perform intrinsic model selection, i.e., to deactivate clusters by assigning exactly zero mass to some columns of $\bT$.

\section{Gromov-Wasserstein based inference for SBM}\label{sec:main}

For several distributions $p(\cdot|\bTheta)$, with learnable $\bTheta \in \R^{K\times K}$, and true $K$ \emph{known}, we show in the next sections that consistent estimates of the SBM parameters can be obtained by solving problem~\eqref{eq:eq_ELBO_srGW} \emph{without} the relative entropy term. However removing entropy is not entirely sufficient to perform model selection (i.e. estimate $K$) and we'll explain why. 
The sparsity promoting penalty introduced in~\cite{srGW}
can be adopted for such a scope.
\subsection{Novel GW-based estimators}

We consider in the following an adjacency matrix $\bA \in \R^{N\times N}$ generated by a SBM with underlying probability density or mass function $p(\cdot | \Theta_{kl})$. Let $\bT\in \mathcal{U}_K(\boldsymbol{1}_N/N)$ and $\bTheta \in \R^{K\times K}$ denote an admissible transport plan and a connectivity matrix, respectively. We consider the GW losses parameterized by an inner loss $\ell$, either ``aligned'' with the generative model, i.e. such that $\ell(A_{ij},\Theta_{kl}) = - \log p(A_{ij} | \Theta_{kl}) $, or taken as any Bregman divergence $D_U(A_{ij} | \Theta_{kl})$ induced by the strictly convex and continuously differentiable function $U$ defined on an appropriated convex set $\mathcal{C}\subset \R$. These two types of GW losses are denoted by
\begin{equation}
L_p(\bT, \bTheta) = {\sum_{ijkl}{ -\log{p(A_{ij} | \Theta_{kl})} T_{ik}T_{jl}}}~,~ \text{and}~,~L_U(\bT, \bTheta) = {\sum_{ijkl}D_U(A_{ij}, \Theta_{kl}) T_{ik}T_{jl}},
\label{global-loss}
\end{equation}
respectively, and associated with  the corresponding semi-relaxed GW divergence
\begin{equation}
    \text{srGW}_{\bullet}(\bA, \bTheta) = \underset{\bT\in \mathcal{U}_K(\boldsymbol{1}_N/N)}{\text{min}}{L_{\bullet}(\bT, \bTheta)}.
    \label{srGW}
\end{equation}
Then, the estimators whose properties we investigate are
\begin{equation}
    \widehat{\bT}, \widehat{\bTheta} \in \underset{\bT\in \mathcal{U}_K(\boldsymbol{1}_N/N), \bTheta}{\text{argmin }}{L_\bullet(\bT, \bTheta)}.
\label{sbm_estimators}
\end{equation}
We formalized the above two families of GW losses since the results we prove in the following might be valid either for one loss type or the other. However, note that the intersection of the two families is not empty: for instance, a Bernoulli SBM coupled with a Bernoulli negative log-likelihood  inner loss $\ell$ is in line with both formulation.

\paragraph{Remarks and relation with prior works} The above estimators fall in the framework of GW minimal estimator discussed in Section~\ref{sec:GW} and can be viewed as M-estimators induced by a modification of the ELBO by removing the relative entropy term or, alternatively, as a continuous relaxation of the maximum likelihood estimator in the framework of~\cite{GAUCHERMLE}.
Moreover, in case of Bernoulli SBM, with inner loss $\ell(\cdot,\cdot) = (\cdot - \cdot)^2$ and \emph{fixed} $\bTheta$, the estimator $\widehat{\bT}$ in Eq.~\eqref{sbm_estimators} reduces to the one used in~\cite{srGW,OTclusteringattributed} for community detection purposes.
Furthermore, as detailed in Appendix~\ref{sec:appendix_various_losses}, for a \emph{fixed} $\bT$, either minimizing $L_U$ w.r.t $\bTheta$ with $\ell$ equal to the square loss in case of a Bernoulli SBM or, alternatively, minimizing $L_p$ w.r.t $\bTheta$ for the Bernoulli, Gaussian or Poisson models leads to the estimator  $\widehat{\Theta}_{kl} := \sum_{ij}{T_{ik}T_{jl}A_{ij}} / \sum_{ij}{T_{ik}T_{jl}}$. Interestingly, this estimator is the same one that appears in the M step of the variational E-M algorithm detailed in~\cite{Daudin}.

\subsection{Theoretical analysis and results of consistency}

 In the first part of this section we assume that the \emph{actual} connectivity $\bTheta^*$ used to generate the data via SBM is known. We show that, in this scenario, the estimated transport plan $\widehat{\bT}$ can consistently estimate the \emph{actual} cluster assignments $Z^*$. To keep the notation uncluttered, we remove the star $^*$ and intend $\bTheta = \bTheta^*$ and $Z = Z^*$, in the following.

\subsubsection{Known connectivity matrix}

 We introduce the following conditional expected loss 
\begin{equation*}
 \mathcal{L}_{\bullet}(\bT, \bTheta) := \E\left[L_{\bullet}(\bT, \bTheta) | Z \right],
\end{equation*}
where, as previously said, $L_{\bullet}$ can be $L_p$ or $L_U$ (same for $\mathcal{L}_{\bullet}$) and state the next

\begin{Lemma}
    Given a Bernoulli SBM and an inner loss taken as a Bregman divergence $D_U$, under mild assumptions on $U$ 
    the transport plan $\bT^*$ minimizing $\bT \mapsto \mathcal{L}_U(\bT, \bTheta)$ is of the form $\displaystyle \bT^* = \frac{1}{N}ZP_\sigma$ where $P_\sigma$ is any permutation matrix such that $\bTheta = P_\sigma \bTheta P_\sigma^\top$.
    \label{prop:expectedloss}
\end{Lemma}

\begin{proof} See Appendix~\ref{proof:Lemma1}.
\end{proof}

\begin{remark}
Being \emph{fixed} in the previous lemma, $\bTheta$ is usually \emph{not} invariant under the action of permutations $P_{\sigma}$ other then the identity matrix. In those cases the reader can safely consider $P_{\sigma} = I_K$ in the previous lemma. However, exceptions exist. Consider for instance the affiliation SBM in which $\bTheta = \eta \mathbf{1}_K\mathbf{1}_K^T + \delta I_K$, where $0<\eta, \delta<1$. In that case, any permutation of rows/columns of $\bTheta$ leaves it invariant.   
\end{remark}

\begin{remark}
To ease the exposition, the previous lemma is formulated under the same assumptions of the next theorem. However, an alternative formulation valid on SBMs other than Bernoulli (exponential family and zero-inflated distributions) and when replacing $\mathcal{L}_U$ with $\mathcal{L}_p$ is stated and proved in Appendix~\ref{proof:alt_Lemma1}). 
\end{remark}
The next theorem states that we can recover the true labels as $N$ goes to $+\infty$ when we know the true connectivity matrix $\bTheta$. The proof is postponed to Appendix~\ref{proof:Th1} and relies on both Lemma~\ref{prop:expectedloss} and tools from high dimensional statistics, in particular concentration results for symmetric random matrices with subgaussian entries. 

\begin{theorem}
    Under a Bernoulli SBM model and with an inner loss taken as a Bregman divergence, let $\widehat{\bT} := \underset{\bT\in \mathcal{U}_K(1/N)}{\text{argmin }} L_U(\bT, \bTheta)$ and      $\bT^*$ defined as in Lemma~\ref{prop:expectedloss}. Then, under the same assumptions as in Lemma~\ref{prop:expectedloss} on $U$, we have
    \begin{equation}
        \| \widehat{\bT} - \bT^*\| \underset{N\to +\infty}{\overset{\P}{\xrightarrow{\hspace{0.9cm}}}} 0. 
    \end{equation}
    Moreover, $\mathcal{L}(\widehat{\bT}, \bTheta) - \mathcal{L}(\bT^*, \bTheta) \underset{N\to +\infty}{\overset{a.s}{\xrightarrow{\hspace{0.9cm}}}} 0$ with rate of convergence $\mathcal{O}(1/\sqrt{N})$.
    \label{labelrecovery}
\end{theorem}

Therorem~\ref{labelrecovery} states that we can consistently recover the cluster assignments $Z$ of the SBM. However we made a strong assumption with the knowledge of $\bTheta^*$. In the next result we prove that the srGW barycenter estimator of $\bTheta^*$, introduced in the previous section, consistently estimates the true connectivity matrix. 

\subsubsection{Unknown connectivity matrix}

We now go back to the situation where $\bTheta$ is unknown and must be inferred from the data. We show that the estimated connectivity $\widehat{\bTheta}$ converges in probability towards the \emph{actual} connectivity matrix $\bTheta^*$ under proper assumptions regarding both the generative SBM and the inner loss $\ell(A_{ij}, \Theta_{kl})$. We first show that the log-likelihood of the observed data under SBM and the srGW loss do behave similarly.

\begin{Lemma}
\label{prop:ineq}
 With inner loss $\ell(\cdot, \cdot) = -\log p(\cdot | \cdot) $, the following inequality holds.
\begin{equation}
\underset{\bTheta}{\sup}{\left| \frac{1}{N^2} \underset{\balpha}{\sup} \log p(\bA | \bTheta, \balpha) - (- \srGW_p(\bA, \bTheta))\right|} \leqslant \frac{\log{K}}{N}
\end{equation}   
\end{Lemma}

\begin{proof}
 See Appendix~\ref{app:proof_Th2}.  
\end{proof}
 The above inequality states that (-srGW) is uniformly closed to the log-likelihood (over $N^2$). Note that the opposite signs for both quantities are in agreement with the fact that they respectively relate to a minimization and a maximization problem. 
 Interestingly, no assumption is made here about the conditional distribution of the adjacency matrix given the labels. A consequence of this inequality, using results from~\cite{ConsistencySBM}, is the convergence in probability of the srGW estimator $\widehat{\bTheta}$ towards the true connectivity matrix in the case of a Bernoulli SBM when the number of nodes $N$ goes to $+\infty$. 

\begin{theorem}\label{Th:Pconv} Under a Bernoulli SBM and assumptions \textbf{A1}-\textbf{A2}-\textbf{A3} in Section~\ref{sec:SBM}, the $\srGW_p$ estimator $\widehat{\bTheta}$ in Eq.~\ref{sbm_estimators} converges in probability to $\bTheta^*$, i.e.  
    \begin{equation}
         \widehat{\bTheta} \underset{N\to +\infty}{\overset{\P}{\xrightarrow{\hspace{0.9cm}}}} \bTheta^*.
    \end{equation}
\end{theorem}

\begin{proof}
See Appendix~\ref{app:proof_Th2}
\end{proof}

\subsection{Optimization and selection of \texorpdfstring{$K$}{}}\label{subsec:optim}
 
\begin{algorithm}[H]\label{GOTSBM}
    \caption{Block Coordinate Descent for Estimation of SBM parameters/clsuters}
Initialize $\bT_0$ via $k$-means or spectral clustering\;
\For{$t=0$ \KwTo N$\_$iter}{  
    $\bTheta_{t} \gets \underset{\bTheta}{\text{argmin }}{L_\ell(\bT_t, \bTheta)}$  \Comment{via closed-form solution in Eq.~\eqref{eq:closed_theta}}\;
    $\bT_{t+1} \gets \underset{\bT}{\text{argmin }}{ \left(L_\ell(\bT, \bTheta_t) + \lambda\Omega(\bT)\right)}$ \Comment{Majorization-Minimization algorithm from \cite{srGW}}\;
}
Return $(\widehat{\bTheta}, \widehat{\bT}) := (\bTheta_t, \bT_{t+1})$\;
\end{algorithm}

So far, we assumed that the actual $K$ (say $K^*$) was known. However, when this is not the case, $K^*$ cannot be estimated by solving $f(K) :=\min_{\bTheta \in \mathbb{R}^{K \times K}}\srGW_{\bullet}(\bA, \bTheta)$. 
Stated differently, for a solution $\widehat{\bT} \in \mathcal{U}_K(N^{-1}\mathbf{1}_N)$, the estimated cluster proportions $\widehat{\balpha}= \widehat{\bT}^\top \mathbf{1}$ cannot be $K^*$-sparse. The reason is that the previous $f(\cdot)$ is monotonically decreasing in $K$ (detailed in Appendix~\ref{app:mod_sel}). For this reason and following~\cite{srGW}, we propose to fix a high initial $K$ in the experiments and adopt the $\ell_{1/2}$ pseudo-norm $\Omega(\bT) := \sum_{k=1}^{K}{\left( \sum_{i=1}^N{T_{ik}}\right)^{1/2}}$, with regularization strength $\lambda \in \R_+$, as shown in Algorithm~\ref{GOTSBM}. That penalty forces the estimated cluster proportions $\widehat{\balpha}$ to be sparse, which comes down to set to $\mathbf{0}_N$ some columns of $\widehat{\bT}$. 
In practice, we focus on \emph{composite inner losses} $\ell$ such that the $\srGW_\ell$ loss is a quadratic function in $\bT$. As detailed in Appendix~\ref{sec:appendix_various_losses}, this covers a range of use cases relevant for SBM, including the square loss and the Bernoulli log-likelihood, for which this assumption on $\ell$ also allows to derive closed-form solutions for $\bTheta$.
In this setting, the OT matrix for the regularized $\srGW_\ell$ problem, and a fixed $\bTheta_t$, is handled with the Majorization-Minimization (MM) approach proposed in \cite{srGW}. This solver iterates over $t$ between the following steps:  (i) linearizing $\Omega$ in the currently estimated $\bT_t$ providing a matrix $R(\bT_t)$; (ii) minimizing in $\bT$ the $\srGW_\ell$ loss regularized with the linear term $\langle R(\bT_t),\bT \rangle $ with for instance a Frank-Wolfe algorithm. While using a warmstart strategy for step (ii), we observed that the computational time for this solver is mostly due to its first iteration, hence it is nearly as efficient as solvers for the unregularized srGW problem.
As shown in Algorithm~\ref{GOTSBM}, the global optimization problem for SBM parameter and cluster estimation can be solved 
by block-coordinate descent, alternating between: (i) updating $\bTheta$ given $\bT$ via the closed-form solutions in Eq.~\eqref{eq:closed_theta} of Appendix~\ref{sec:appendix_various_losses})
; (ii) updating $\bT$ given $\bTheta$ using the MM algorithm detailed above. Overall, this solver can easily be implemented on GPU and its theoretical complexity is dominated by the computation of the gradient w.r.t $\bT$, with cost $O(N^2K + K^2N)$.

\section{Numerical experiments}\label{sec:exp}

Whereas the potential of srGW in terms of community detection in real networks has been assessed in \cite{srGW,OTclusteringattributed}, mostly with a fixed atom $\bTheta$ equal to $I_K$ in order to enforce assortativity,  in this section we evaluate our approach on synthetic data generated from SBM, in order to empirically validate the theoretical results in the previous section. We consider three classical SBM regimes capturing different connectivity patterns: \emph{Assortative}, \emph{Hub}, and \emph{Disassortative}. 
Let $\alpha, \beta \in \left[0, 1\right]$, 
such that $\alpha > \beta$. In the \emph{Assortative} case, the connectivity matrix is of the form $\bTheta  = (\alpha - \beta)I_K + \beta \boldsymbol{1}\boldsymbol{1}^T $. In the \emph{Hub} case, we consider the same $\bTheta$ except for the first row and column, which are set to $\alpha$ instead of $\beta$. In the \emph{Disassortative} case, the connectivity matrix is of the form $\bTheta  = (\beta - \alpha)I_K + \alpha \boldsymbol{1}\boldsymbol{1}^T$.
In each scenario, we generate graphs with $N=10^3$ nodes, fixing $\beta = 0.03$ and varying $\alpha \in [\beta, 0.2]$. Some of these samples are illustrated in Figure~\ref{fig:enter-label} of Appendix~\ref{app:exp}. 

\paragraph{Partitioning} 
\begin{figure}[t!]
        \centering
        \makebox[\textwidth]{%
    \includegraphics[width=1.0\textwidth]{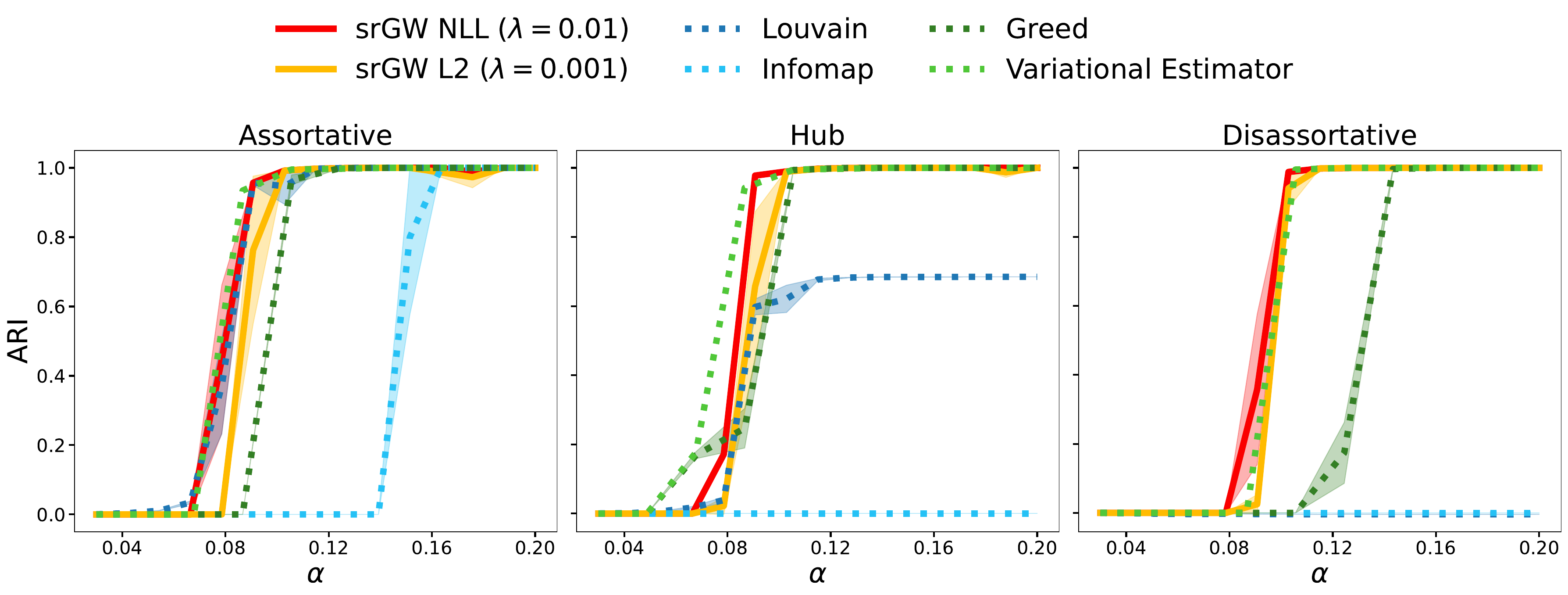}
}
\vspace{-7mm}\caption{Evolution of the average ARI  w.r.t $\alpha$, 5 simulated graphs for each $\alpha$. The higher $\alpha$, the easier to detect the clusters.}
\label{ARI}
\end{figure}
\begin{figure}[t!]
  \centering
          \makebox[\textwidth]{%
    \includegraphics[width=1.\textwidth]{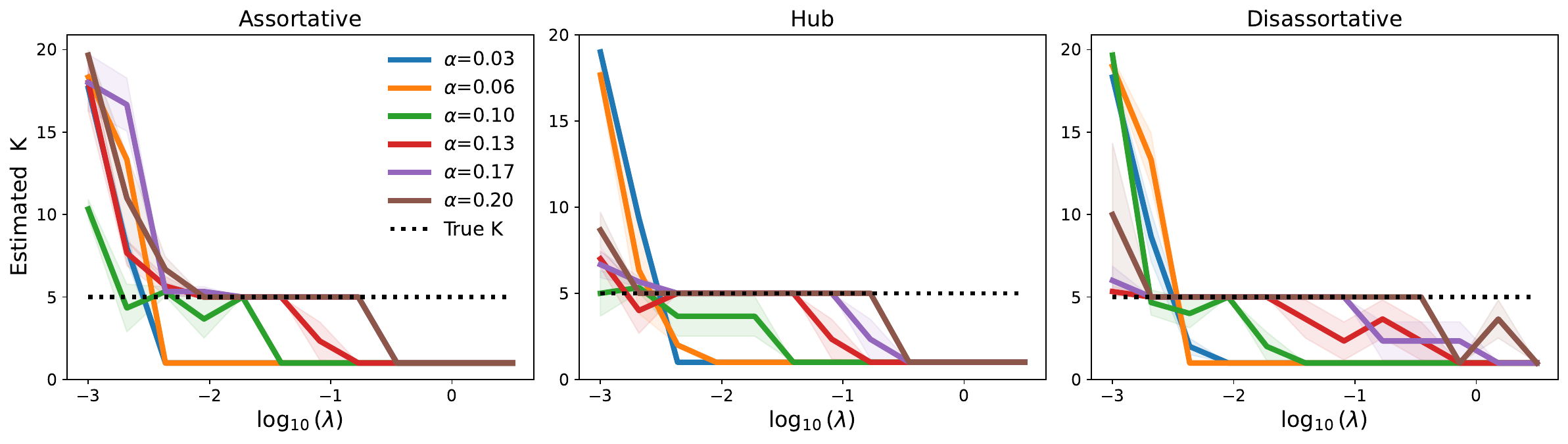}
    \label{plateau}
}
\vspace{-7mm}\caption{Average selected $K^*$ versus (log of) the sparsity hyper-parameter $\lambda$.}
\end{figure}

We first evaluate the ability of srGW-based estimators to recover the latent clusters across the three SBM regimes. We consider two inner losses for srGW: the binary cross-entropy (srGW-NLL) and the squared loss (srGW-L2). For both methods, the transport plan $\bT_0$ is initialized from spectral clustering based on the leading singular vectors of $\bA$. Empirically, this initialization consistently outperformed alternatives such as random initialization or direct $k$-means on the rows of $\bA$. The sparsity parameter $\lambda$ was selected empirically for both methods. We compare these two variants against four state-of-the-art baselines. The first two are \emph{Louvain}~\cite{Louvain} and \emph{Infomap}~\citep{rosvall2009map}, two widely used community detection methods respectively based on modularity maximization and flow compression. We also consider two model-based competitors: \emph{Greed}~\cite{Greed21}, which performs greedy maximization of the Integrated Classification Likelihood (ICL), and the \texttt{blockmodels} package~\cite{blockmodels}, which implements variational EM with ICL-based model selection. Except for Louvain and Infomap, all methods search over at most $K=20$ clusters, whereas the true number of clusters is set to $K^*=5$. Partitioning performance is measured using the Adjusted Rand Index (ARI)~\cite{AdjustedMetricComparison, ComparingPartition}, which takes values in $[0,1]$, where 1 indicates perfect recovery of the true labels (up to label permutation), and 0 corresponds to random assignments. The average ARI over 5 simulated graphs for each value of $\alpha$ is reported in Figure~\ref{ARI}. In the Assortative and Hub regimes, srGW-NLL consistently recovers the hidden clustering structure more accurately than srGW-L2, highlighting the benefit of a likelihood-based loss. As expected, Louvain and Infomap perform competitively in the assortative setting but degrade markedly in the Hub and Disassortative regimes, where their underlying structural assumptions are violated. Most importantly, we observe that srGW-NLL compares favorably to the best competitors, Greed and VEM, while being 1 to 2 orders of magnitude faster on CPU, as reported in Appendix~\ref{runtimes}.

\paragraph{Sensitivity analysis} 
 We analyze the sensitivity of srGW NLL to the hyperparameter $\lambda$, in the same settings as above. Figure~\ref{plateau} shows the average estimated number of clusters over 5 runs as a function of~$\lambda$. Interestingly, as soon as $\alpha$ is high enough to permit the actual clusters to be recovered, we see a plateau taking form around $10^{-2}$, roughly corresponding to a value of $\lambda = K/2N$.    




\section{Conclusion and perspectives}

We have shown that minimizing an optimal transport divergence leads to consistent estimators of the SBM connectivity matrix and clusters of nodes.  
These estimators can be efficiently and simply computed through a block-coordinate descent. A penalization on the transport plan is used to perform model selection. The tuning of this hyperparameter should be investigated in more depth, in the future. Further work could include comparison with MLE, variational and spectral estimators and explore minimax optimality of srGW barycenter estimators, using in particular techniques similar to \cite{GaucherSBM}. A line of interest would be to show that when we do not know $K^*$, the obtained target structure is asymptotically weakly isomorphic to the true connectivity matrix. 


\clearpage

\bibliographystyle{unsrtnat}
\bibliography{references}

\medskip


\clearpage

\appendix

\section{Proofs of consistency}

\subsection{Proof of Lemma~\ref{prop:expectedloss}}\label{proof:Lemma1}

Throughout this section it is always assumed that $\bTheta$ is the true connectivity matrix and $Z$ the actual cluster assignment of the SBM the adjacency matrix is sampled from. In order to keep the notation uncluttered we use $\bTheta$ in place of $\bTheta^*$ and $Z$ in place of $Z^*$.
In order to prove Lemma~\ref{prop:expectedloss} we need the following result 

\begin{lemma} \label{lemma:minimizers} Let 
$d:\R^2\mapsto \R_+$ a divergence, i.e. such that $d(a, b) = 0 \Longleftrightarrow a = b$, $Z :\llbracket N \rrbracket \to \llbracket K \rrbracket$ a surjective node assignment, and $\bTheta \in \R^{K\times K}$ satisfying assumption (\textbf{A1}) in Section~\ref{sec:SBM}.
Finally, let $F :\bT \in \mathcal{U}_K(\boldsymbol{1}_N/N) \mapsto \displaystyle \sum_{i,j,k,l}d(\Theta_{Z_i Z_j}, \Theta_{kl})T_{ik} T_{jl}$. Then, the minimizers of $F$ are exactly of the form $T = \frac{1}{N}ZP_\sigma$ where $P_\sigma$ is any permutation matrix, possibly limited to the identity, such that $\bTheta = P_\sigma \bTheta P_\sigma^\top$.
\end{lemma}

\begin{proof}[Proof of Lemma~\ref{lemma:minimizers}]
Since the above quantity is positive the best we can do to minimize it in $\bT$ is to chose it such that $F(\bT) = 0$. First, we remark that if $\bT = \frac{1}{N}Z$, $F(\bT) = 0$. Let $T$ be a minimizer, so that $F(\bT) = 0$. We would like to show that if $Z_i \neq Z_j$, then there is no $k$ such that  $T_{ik}\cdot T_{jk} > 0$. By contradiction, assume that for one pair $(i,j)$ we have $Z_i \neq Z_j$ and $\exists k$, $T_{ik} \cdot T_{jk}>0$. 
Consider a third node $u \in \llbracket N \rrbracket$ and, without loss of generality, assume $T_{ul} > 0$ for some $l$, possibly different from $Z_u$.
Then, since $F(\bT)=0$ one must have $d(\Theta_{Z_iZ_u}, \Theta_{kl})T_{ik}T_{ul} = d(\Theta_{Z_jZ_u}, \Theta_{kl})T_{jk}T_{ul} = 0$. Given that $T_{ik}, T_{jk}$ and $T_{ul}$ are all positive, this forces
 $\Theta_{Z_iZ_u} = \Theta_{kl} = \Theta_{Z_j Z_u}$. Similarly it can be shown that $\Theta_{Z_uZ_i} = \Theta_{l k} = \Theta_{Z_u Z_j}$.
Since $u$ and its clusters $Z_u$ are entirely generic, the last equalities imply that $\forall r \in \llbracket K \rrbracket, \Theta_{Z_i r} = \Theta_{Z_j r}, \Theta_{r Z_i} = \Theta_{ r Z_j}$ which contradicts assumption (\textbf{A1}). Then, it's easy to check we can build $\sigma \in \mathfrak{S}_K$ such that $\bT = \frac{1}{N}ZP_\sigma$ and $\bTheta = P_\sigma \bTheta P_\sigma^\top$.  

\end{proof}

We are now ready to prove the main Lemma. 

\begin{proof}[Proof of Lemma~\ref{prop:expectedloss}.]
We recall the objective
\begin{equation*}
\begin{split}
    \mathcal{L}_U(\bT) &= \E[L_U(\bT, \bTheta) | Z ] = \sum_{ijkl}{\E \left[ D_U(A_{ij}, \Theta_{kl}) | Z \right] T_{ik}T_{jl}}.  
\end{split}
\end{equation*}
Because of the constraints on $\bT$, the following term is constant w.r.t. $\bT$:
\begin{equation*}
\begin{split}
    &\sum_{ijkl}{\E \left[ D_U(A_{ij},\Theta_{Z_iZ_j}) | Z \right] T_{ik}T_{jl}}\\
    &= \frac{1}{N^2} \sum_{ij}{\E \left[ D_U(A_{ij},\Theta_{Z_iZ_j}) |Z \right]} \\
    &= \frac{1}{N^2} \sum_{ij}\left( \E \left[ U(A_{ij}) - U(\Theta_{Z_i Z_j})|Z  \right] - U'(\Theta_{Z_i Z_j}) \E \left[ A_{ij} - \Theta_{Z_i Z_j}|Z  \right] \right).
\end{split}
\end{equation*}
Then, assuming that $\E \left[ A_{ij} |Z  \right] = \Theta_{Z_i Z_j}$, like in Bernoulli or Poisson SBMs, the last term inside parentheses vanishes and  minimizing $\mathcal{L}(\bT)$ boils down to minimizing 
\begin{equation*}
\begin{split}
    F(\bT) :&= \sum_{i,j,k,l}{\E[ D_U(A_{ij}, \Theta_{kl}) - D_U(A_{ij}, \Theta_{Z_iZ_j}) | Z]T_{ik}T_{jl}} \\
    &= \sum_{i,j,k,l}{\E[ U(\Theta_{Z_iZ_j}) - U(\Theta_{kl}) - U'(\Theta_{kl})(A_{ij} - \Theta_{kl})) | Z]T_{ik}T_{jl}} \\
    &= \sum_{i,j,k,l}{(U(\Theta_{Z_iZ_j}) - U(\Theta_{kl}) - U'(\Theta_{kl})(\Theta_{Z_iZ_j} - \Theta_{kl}))T_{ik}T_{jl}} \\
    &= \sum_{i,j,k,l}{D_U(\Theta_{Z_iZ_j}, \Theta_{kl})T_{ik}T_{jl}}.
\end{split}
\end{equation*}
We can now apply Lemma~\ref{lemma:minimizers} to obtain the unique minimizer $\bT = \frac{1}{N}ZP_\sigma$ where $P_{\sigma}$ is any permutation such that $\bTheta = P_\sigma \bTheta P_\sigma^\top$.
\end{proof}

\subsection{An alternative version of Lemma~\ref{prop:expectedloss}}\label{proof:alt_Lemma1}

In Lemma~\ref{prop:expectedloss}, we assumed that inner loss $\ell$ of the srGW problem in Eq.~\eqref{eq:app_objective} was a Bregman divergence $D_U$ and the generative model a Bernoulli SBM.
However, in case the inner loss $\ell$ is the negative logarithm of the same probability distribution used in the generative SBM (what we called $L_p(\bT, \bTheta)$ in Eq.~\eqref{global-loss}), Lemma~\ref{prop:expectedloss} always holds, no matter the type of generative SBM (i.e. Bernoulli, Poisson, Gaussian, zero-inflated laws).

\begin{proof}[Proof of Proposition \ref{prop:expectedloss}.]
Recall that our objective now is
\begin{equation}
\mathcal{L}_p(\bT) = \E[L_p(\bT, \bTheta) | Z ] = \sum_{ijkl}{\E \left[ -\log p(A_{ij} | \Theta_{kl}) | Z \right] T_{ik}T_{jl}}
\label{eq:app_objective}
\end{equation}
where the conditional expectation is taken under $A_{ij} \sim p(\cdot | \Theta_{Z_i Z_j})$. Observing that the following term is a constant given constraints on $\bT$,
\begin{equation*}
\sum_{ijkl}{\E \left[ -\log p(A_{ij} | \Theta_{Z_iZ_j}) | Z \right] T_{ik}T_{jl}} = \frac{1}{N^2}\sum_{ij}\E \left[ -\log p(A_{ij} | \Theta_{Z_iZ_j}) | Z \right],
\end{equation*}
    
it can readily be seen that minimizing the above objective w.r.t. $\bT$ boils down to minimizing
\[
F(\bT) = \sum_{i,j,k,l}\mathbb{E}\left[\log \frac {p(A_{ij}|\Theta_{Z_i Z_j})}{p(A_{ij}|\Theta_{kl})} \bigm\vert Z \right]T_{ik}T_{jl} = \sum_{i,j,k,l} d(\Theta_{Z_i Z_j}, \Theta_{kl})T_{ik} T_{jl}
\]
where $d$, here, is nothing but the conditional KL divergence between $p(\cdot | \Theta_{Z_iZ_j})$ and $p(\cdot | \Theta_{kl})$, being null iff $\Theta_{Z_i Z_j} = \Theta_{kl}$\footnote{Recall that, in this section, $Z$ is used in place of $Z^*$. So $p(A_{ij}| \Theta_{Z_iZ_j})$ is a short hand notation for $p(A_{ij} | (Z_i,Z_j)= (Z^*_i, Z_j^*); \Theta^*_{kl})$ whereas $p(A_{ij}| \Theta^*)$ stands for $p(A_{ij} | (Z_i,Z_j) = (k,l); \Theta^* )$}.
Now, thanks to Assumption~(\textbf{A3}) in Section~\ref{sec:SBM}, $Z$ is a surjective node assignment and all assumptions of Lemma~\ref{lemma:minimizers} are fulfilled, so we have the result.
\end{proof}

To summarize, we observe that the two versions of Lemma~\ref{prop:expectedloss} differ in two aspects:
    \begin{enumerate}
        \item the alternative version is true for all SBMs whereas the original only holds for models which satisfy $\E[A_{ij} | Z] = \Theta_{Z_iZ_j}$ (the Bernoulli model is a specific case here);
        \item the proof of proposition of the alternative version exploits the KL divergence, whereas the one of the original version exploits the definition Bregman divergence.
    \end{enumerate}
Moreover, a summary of what proved as preliminar to the statement of Theorem~\ref{labelrecovery} cans be seen in Figure~\ref{fig.Tree}.

\begin{figure}[t]
\centering
\begin{tikzpicture}[
  node distance=2cm,
  every node/.style={draw, rectangle, rounded corners, align=center},
  arrow/.style={->, thick}
]

\node (L1) {Lemma~\ref{lemma:minimizers}};

\node (P1) [below left of=L1, xshift=-2.5cm] {Lemma~\ref{prop:expectedloss} (Bernoulli with Bregman divergences)};
\node (P1b) [below right of=L1, xshift=2.5cm] {Lemma~\ref{prop:expectedloss} bis (Negative Log Likelihood)};

\node (T1) [below of=P1, xshift=2cm] {Theorem~\ref{labelrecovery} (Bernoulli with Bregman divergences)};

\draw[arrow] (L1) -- (P1);
\draw[arrow] (L1) -- (P1b);

\draw[arrow] (P1) -- (T1);
\end{tikzpicture}

\caption{Tree structure of our results up to Theorem~\ref{labelrecovery}.}
\label{fig.Tree}
\end{figure}
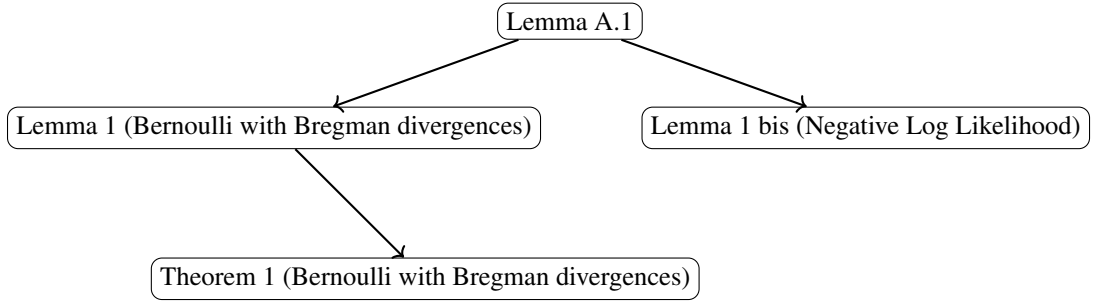    

\clearpage



    

\subsection{Recovery of the labels (Theorem \ref{labelrecovery})}\label{proof:Th1}

In the following, we prove Theorem~\ref{labelrecovery}that guarantees the recovery of the node labels ($Z$) via  srGW based estimator, \emph{knowing the true connectivity matrix} $\bTheta=\bTheta^*$. 
This result holds when the inner loss $\ell$ is any Bregman divergence $D_U$,  induced by U such that
\begin{equation*}
    \ell(a,b) = D_U(a,b) = U(a) - U(b) - U'(b)(a-b),
\end{equation*}
leading to the srGW loss denoted $L_U$ in Eq.~\ref{global-loss}. 
For $D_U$ to be a well-defined divergence for Bernoulli SBM, we need the following assumption:
\paragraph{Assumption A.4\label{assumption_U}} We assume that $U:\left[0;1 \right] \rightarrow \R$ is strictly convex on $\left[0;1 \right]$ and continuously differentiable at least on $\left] 0;1 \right[$. If it is on $\left[ 0;1 \right]$, no additional assumption is needed. Otherwise, we leverage Assumption \textbf{A.2}, constraining $\Theta_{kl}$ to $\left[\zeta, 1-\zeta \right]$ with $\zeta > 0$, such that in all cases there exists $c>0$ satisfying $|U^\prime(\Theta_{kl})| \leqslant c$.


 \begin{remark}
    Taking the square loss or the negative Bernoulli log-likelihood as losses to estimate the labels is covered by our choice of Bregman divergence $D_U$.
    Indeed, the square loss is recovered by taking $U : x \mapsto x^2$, and its derivative $U'(x) = 2x$ is bounded by 2 for any $\Theta_{kl} \in \left[0, 1\right]$. Then for the negative Bernoulli log-likelihood, one can consider $U : x \mapsto x\log{x} + (1-x)\log{(1-x)}$ such that $\displaystyle \ell(x, y) = x\log{\frac{x}{y}} + (1-x)\log{\frac{(1-x)}{(1-y)}} $. In this setting, the elements which depend only on $x$, would be ignored in the $\srGW_\ell$ problem because of the constraint on the fixed marginal $N^{-1}\mathbf{1}_N$. Therefore the aforementioned Bregman divergence $\ell$ can be considered as equivalent to the Binary Cross Entropy loss $\ell_{BCE}(x, y) = -x\log{y} - (1-x)\log{(1-y)}$, usually encountered for Bernoulli SBM, as they both share the same minimizers for their corresponding srGW problem. Moreover, under assumption (\textbf{A2}), $\Theta_{kl} \in \left[\zeta, 1-\zeta\right]$, for which $U'(x) = \log\frac{x}{1-x}$ is bounded by $U'(1-\zeta)$.  \\
\end{remark}

We also recall a classical concentration result used in the following proof and whose proof can be found in~\cite{Vershynin_2026}.

\begin{theorem}(Concentration of operator norm of symmetric matrices with subgaussian entries)\label{Th:opnorm}
    Let $\boldsymbol{M}$ be a $N\times N$ matrix with subgaussian entries on and above the diagonal. Then, for any $t>0$ we have :
    \begin{equation*}
\|\boldsymbol{M}\|_{\text{op}} \leqslant CD(\sqrt{N} + t)
    \end{equation*}

with probability at least $1 - 4e^{-t^2}$. Here, the constant $C$ is an absolute constant, and $D=\max_{ij}{\|M_{ij}\|_{\psi_2}}$, where $\|\cdot \|_{\psi_2}$ is the notation for Orlicz norm of subgaussian random variables.
\end{theorem}


\begin{proof}[Proof of Theorem~\ref{labelrecovery}]
The minimization problem we consider is
\begin{equation}
    \min_{T \in \mathcal{U}_K(\boldsymbol{1}_N/N)} \underbrace{\sum_{i, j}^{N} \sum_{k,l}^K D_U(A_{ij}, \Theta_{kl})T_{ik}T_{jl}}_{L_U(\bT, \bTheta)},
    \label{eq:main}
\end{equation}
where $A$ is the random adjacency matrix coming from a Bernoulli SBM and $\bTheta$ is assumed to be the \emph{actual} connectivity matrix in SBM. The first quantity we consider is 
\[
\mathcal{L}_U(\bT,\bTheta) := \mathbb{E}\left[L_U(\bT,\bTheta) | Z \right]
\]
where the expectation is conditional given $Z$, i.e. the actual labels used to sample $\bA$ via SBM \emph{and} the actual $\bTheta$. In other words: $\mathbb{E}[A_{ij}|Z] = \Theta_{Z_i Z_j}$. 

We recall from proposition \ref{prop:expectedloss} that the minimizers of $\mathcal{L}_U(\bT, \bTheta)$ are of the form $T^{*}:=\frac{1}{N}ZP_\sigma$ where $\sigma \in \mathfrak{S}_K$.


We take again the function $F_N$ defined in lemma 1 (notice that $F_N(T) = \mathcal{L}(T) - \mathcal{L}(T^*)$).


Here the space of the transports depends on $N$, indeed it is $\mathcal{U}_K(\mathbbm{1}_N/N)$. Hence, the convergence theorem of $M$ estimator doesn't hold. We have to use instead a sequence of functions $F_N$ each defined on $\mathcal{U}_K(\mathbbm{1}_N/N)$, and unify them with a projection $\pi_N : \mathcal{U}_K(\mathbbm{1}_N/N) \to \R^{K\times K}$.

Let $\displaystyle \pi_N : T \mapsto \left( \sum_{Z_i=k}{T_{il}} \widehat{\alpha}_k^{-1} \right)_{kl}$, a ``projection" of the transport plan into the set matrices with positive entries such that lines sum to one, where $\displaystyle\widehat{\alpha}_k := \frac{1}{N}\sum_{i}Z_{ik}$.

Let $f(B) = \sum_{kk'll'}{ B_{kk'}B_{ll'}d(\Theta_{kl}, \Theta_{k'l'})}$ which is continuous on the set $\mathcal{U
}_K(\mathbbm{1}_K)$  wich is compact, and $f_N(B) = \sum_{kk'll'}   {\widehat{\alpha}_k\widehat{\alpha}_l B_{kk'}B_{ll'}d(\Theta_{kl}, \Theta_{k'l'})}$. A good property of $f_N$ is that $f_N \circ \pi_N$ is related to $F_N$ through the inequality $F_N \geqslant f_N \circ \pi_N - C/N$ with a constant $C>0$, independent of $N$ \footnote{The exact constant is $C = \Delta(\bTheta) K^2$ where $\Delta(\bTheta)  = \max_{k,k',l'}{d(\Theta_{kk}, \Theta_{k'l'})} > 0$ }.

We start by fixing the threshold $\delta > 0$. Now, we fix $\displaystyle \eta < \gamma/2$, and we will work under the event $\mathcal{E} = \Big(\max_{k}{|\alpha_k - \widehat{\alpha}_k | \leqslant \eta}\Big)$. Notice this implies $\displaystyle f_N \geqslant \frac{\gamma^2}{4} f$.

Now, there exists $\varepsilon := \varepsilon(\delta) > 0$ such that 

\begin{equation}
    \underset{\|B - I_K\|_1 \geqslant \delta}{f(B)} \geqslant \varepsilon
\end{equation}


\begin{equation*}
    \begin{split}
    \|T - Z/N\|_1 &= \sum_{i, k}{|T_{ik} - \delta_{kZ_i}/N|}  \\
    &= \sum_{k, l}\sum_{Z_i=l}{|T_{ik} - \delta_{kl}/N|} \\
    &= \sum_{k\neq l}\sum_{Z_i=l}{T_{ik}} + \sum_{k, Z_i=k}{\frac{1}{N} - T_{ik}} \\
    &=\sum_{k\neq l}{\widehat{\alpha}_l\pi_N(T)_{lk}} + \sum_{k, Z_i=k}{\frac{1}{N} - T_{ik}} \\
    &=\sum_{k\neq l}{\widehat{\alpha}_l\pi_N(T)_{lk}} + \sum_{k}{\widehat{\alpha}_k(1 - \pi_N(T)_{kk})} \\
    &\leqslant \|\pi_N(T) - I_K\|_1 \text{ because each } \widehat{\alpha}_k \leqslant1 \\
    \end{split}
\end{equation*}

Then $\delta \leqslant \|T - Z/n\|_1 \Longrightarrow \delta \leqslant \|\pi_N(T) - I_K\|_1 $. So $f(\pi_N(T)) \geqslant \varepsilon$.

Then $\displaystyle\varepsilon \leqslant f(\pi_N(T))  \Longrightarrow \varepsilon \frac{\gamma^2}{4} \leqslant f_N(\pi_N(T)) \leqslant F_N(T) + \Delta(\bTheta)K^2/N$ where $\Delta(\bTheta) = \max_{k,k',l'}{d(\Theta_{kk}, \Theta_{k'l'})}$. Taking $N$ large enough so that $ \Delta(\bTheta)K^2/N \leqslant \varepsilon\gamma^2/8$, we get that : $\delta \leqslant \|T - Z/N\|_1$ implies that $F_N(T) \geqslant \varepsilon \gamma^2/8$.

For all $k\in \llbracket K \rrbracket$, let $\mathcal{E}_k = (| \alpha_k - \widehat{\alpha}_k| \leqslant \eta) $, which is $Z$ measurable.

$\P(\mathcal{E}_k \mid Z) = \E[\mathbbm{1}_{\mathcal{E}_k }\mid Z] = \mathbbm{1}_{\mathcal{E}_k } = \mathbbm{1}{\Bigg\{\left|\alpha_k - N^{-1}\sum_{i=1}^N{Z_{ik}}\right| \leqslant \eta\Bigg\}}$.

Now, we'll see that we can control the event $\Big\{ \|T - Z/N\|_1\geqslant \delta \Big\}$ on the event $\displaystyle \bigcap_{k=1}^K{\mathcal{E}_k}$ through a basic probability lemma presented hereafter.

\paragraph{Probability lemma}

Let $\IB$ such that $\P(\IB)\geqslant 1 - \textcolor{Purple}{t}$. $\P(\IA^c \cup \IB^c) \leqslant \P(\IA^c) + \P(\IB^c) \leqslant \P(\IA^c) +  \textcolor{Purple}{t} = 1- \P(\IA) +  \textcolor{Purple}{t}$. Then, $\P(\IA\cap \IB) \geqslant \P(\IA) -  \textcolor{Purple}{t}$. Then $ \P(\IA) \leqslant \P(\IA \cap \IB ) +  \textcolor{Purple}{t} \leqslant \P(\IC) +  \textcolor{Purple}{t}$ for any $\IC$ such that $\IA\cap \IB \subseteq \IC$.

Then, applying this simple lemma yields : 

\begin{equation*}
    \begin{split}
       \P\left( \textcolor{Red}{\|T - Z/N\|_1 \geqslant \delta} \mid Z \right) &\leqslant \P\left( \textcolor{Red}{\|T - Z/N\|_1 \geqslant \delta}, \textcolor{Blue}{\bigcap_{k=1}^{K}\mathcal{E}_k} \mid Z\right) +  \textcolor{Purple}{c_n} \\
       &\leqslant \P\left(\textcolor{Green}{ \mathcal{L}_N(T)- \mathcal{L}_N(T^*) \geqslant \varepsilon \gamma^2 / 8} \mid Z \right) +  \textcolor{Purple}{c_n}
    \end{split}
\end{equation*}

Where $\displaystyle c_n = \sum_{k=1}^K{\P(\mathcal{E}_k^c \mid Z) } = \sum_{k=1}^K{\mathbbm{1}_{\mathcal{E}_k^c}} \xrightarrow[n \to \infty]{\text{a.s}} 0$ by the Strong Law of Large numbers.

We can even have $\|T - Z/N\|_1$ goes to zero almost surely with the Borel-Cantelli's lemma.

Indeed, $\P(\|T - Z/N\|_1 \ge \delta ) \leqslant \P\left( \mathcal{L}(T) - \mathcal{L}(T^*)\geqslant \varepsilon(\delta) \gamma^2/8 \right) + \E(c_n)$.

and we can apply Hoeffding's inequality on $\E(c_n)$ which is nothing else than $\sum_{k=1}^K{\P(\mathcal{E}_k^c)}$.

Now, we have

\begin{equation}
\mathbb{P}\left( \Vert \widehat{\bT} - \bT^*\Vert > \delta \Big| Z \right) \leqslant \mathbb{P} \left( \mathcal{L}_U(\widehat{\bT}, \bTheta) - \mathcal{L}_U(\bT^*, \bTheta) > \varepsilon \gamma^2/8 \Big| Z \right) + c_n.
\label{eq:first_bound}
\end{equation}
Then
\begin{equation*}
    \begin{split}
    \mathcal{L}_U(\widehat{\bT}, \bTheta) - \mathcal{L}_U(\bT^*, \bTheta) &= \mathcal{L}_U(\widehat{\bT}, \bTheta) - L_U(\widehat{\bT},\bTheta) + \underbrace{L_U(\widehat{\bT},\bTheta) - L_U(\bT^*, \bTheta)}_{\leqslant0} \\ &+ L_U(\bT^*,\bTheta) - \mathcal{L}_U(\bT^*, \bTheta) \\
    & \leqslant\mathcal{L}_U(\widehat{\bT}, \bTheta) - L_U(\widehat{\bT},\bTheta) + L_U(\bT^*,\bTheta) - \mathcal{L}_U(\bT^*, \bTheta) \\
    & \leqslant|\mathcal{L}_U(\widehat{\bT}, \bTheta) - L_U(\widehat{\bT},\bTheta)| + |L_U(\bT^*,\bTheta) - \mathcal{L}_U(\bT^*, \bTheta)| \\
    & \leqslant2\left(\sup_{\bT} |L_U(\bT, \bTheta) - \mathcal{L}_U(\bT, \bTheta)|\right).
    \end{split}
\end{equation*}

Thus, combining the above lines with Eq.~\eqref{eq:first_bound}, one obtains
\begin{equation}
 \mathbb{P}\left( \Vert \widehat{\bT} - \bT^*\Vert > \delta \Big| Z \right) \leqslant\mathbb{P}\left(\sup_{\bT} |L_U(\bT, \bTheta) - \mathcal{L}_U(\bT, \bTheta)|>\frac{\varepsilon \gamma^2}{8} \Big| Z \right) + c_n. 
    \label{eq:second_bound}
\end{equation}

Now, we would like to bound the the right-hand term of the inequality above. We introduce the random matrices $\boldsymbol{M}$, $\boldsymbol{M}_U$ such that $M_{ij} = A_{ij} - \E[A_{ij} | Z], M_{U,ij} = U(A_{ij}) - \E[U(A_{ij} )| Z]$. These matrices are symmetric and have subgaussian entries under $\P^* = \P(\cdot | Z)$. This is obvious for $\boldsymbol{M}$, and for $\boldsymbol{M}_U$, note that $U(A_{ij})$ takes only two values $U(0)$ and $U(1)$. Hence, $\boldsymbol{M}_U$ has indeed subgaussian entries. 

\begin{equation*} 
    \begin{split}
     &|L_U(\bT,\bTheta) - \mathcal{L}_U(\bT, \bTheta)| \\ &= \left| \sum_{ijkl} (U(A_{ij})-\E[U(A_{ij})|Z]) + U'(\Theta_{kl})(A_{ij} - \E[A_{ij}|Z])) T_{ik} T_{jl} \right| \\
     & \leqslant  \left| \sum_{ijkl} (U(A_{ij})-\E[U(A_{ij})|Z])  T_{ik} T_{jl} \right| + \left| \sum_{ijkl} (A_{ij}-\E[A_{ij}|Z]) U'(\Theta_{kl})  T_{ik} T_{jl} \right| \\
     &\leqslant \frac{1}{N^2}\left| \underbrace{\sum_{ij}{(U(A_{ij}) - \E[U(A_{ij})|Z])}}_{=\mathbf{1}^T_N \boldsymbol{M}_U \mathbf{1}_N} \right| + c \sum_{kl} \left|\underbrace{ \sum_{ij}{T_{ik}T_{jl}(A_{ij}-\E[A_{ij} |  Z])}}_{= \bT_k^\top \boldsymbol{M} \bT_l}\right|  \\
     &\leqslant \frac{1}{N}\|\boldsymbol{M}_U\|_{\text{op}}  + c \sum_{kl}{\|\bT_k\| \|\bT_l\| \| \boldsymbol{M} \|_{\text{op}}} \\
     & \leqslant \frac{1}{N}\|\boldsymbol{M}_U\|_{\text{op}} + c \frac{K^2}{N} \| \boldsymbol{M} \|_{\text{op}} \text{ because } \forall k \|\bT_k\|\leqslant \frac{1}{\sqrt{N}} \\
    \end{split}
    \end{equation*}

In the second inequality, we used assumption the fact that $|U'(\Theta_{kl})| \leqslant c$. In the case of Bernoulli SBM, we recall that it is implied by assumption (\textbf{A2}).

Therefore  

\begin{equation*}
\begin{split}
    \sup_{\bT \in \mathcal{U}_K(\boldsymbol{1}_N/N)} |L_U(\bT, \bTheta) - \mathcal{L}_U(\bT, \bTheta)| 
    &\leqslant\frac{1}{N}\|\boldsymbol{M}_U\|_{\text{op}} + c \frac{K^2}{N} \| \boldsymbol{M} \|_{\text{op}} \\
    & \lesssim \frac{1}{N}(\sqrt{N} + t ) + cK^2\frac{1}{N}(\sqrt{N}+t)\\
& \lesssim \frac{1}{\sqrt{N}}  +  \frac{t}{N}
    \end{split}
\end{equation*}
    where the last inequality holds under $\P^*$  with probability at least $1-8e^{-t^2}$ according to Theorem~\ref{Th:opnorm}. Note that the constants hidden behind $\lesssim$ are absolute constants and do not depend on $Z$ so that the last inequality holds under $\P$ with probability at least $1-8e^{-t^2}$. Therefore, if we set $\displaystyle \frac{\varepsilon\gamma^2}{8} = C\left(\frac{1}{\sqrt{N}} + \frac{t}{N} \right)$, for some positive constant $C$, we get $t = N\left(\frac{\gamma^2\varepsilon}{16C} - \frac{1}{\sqrt{N}}\right)$ and 

 \begin{equation*}
     \begin{split}
     \mathbb{P}\left(\sup_{\bT \in \mathcal{U}_K(\boldsymbol{1}_N/N)} |L_U(\bT, \bTheta) - \mathcal{L}_U(\bT, \bTheta)|>\frac{\varepsilon\gamma^2}{8} \right) &\leqslant1 - \left[1 - 8\exp{\left(-N^2\left(\frac{\gamma^2\varepsilon}{16C} - \frac{1}{\sqrt{N}}\right)^2\right)}\right] \\
     & = 8\exp{\left(-N^2\left(\frac{\gamma^2\varepsilon}{16C} - \frac{1}{\sqrt{N}}\right)^2\right)}.
     \end{split}
 \end{equation*}
    
    Finally, $\displaystyle\mathbb{P}\left(\sup_{\bT \in \mathcal{U}_K(\boldsymbol{1}_N/N)} |L_U(\bT, \bTheta) - \mathcal{L}_U(\bT, \bTheta)|>\frac{\varepsilon\gamma^2}{8} \right) \underset{N\to +\infty}{\xrightarrow{\hspace{0.9cm}}} 0$. 
    
    Thanks to Borel-Cantelli's Lemma, we even have almost sure convergence, and the rate of convergence is $\mathcal{O}(1/\sqrt{N})$, similar to generalization bound of learning algorithms.
\end{proof}

\subsection{Estimation of \texorpdfstring{$\bTheta$}{} \label{app:proof_Th2}}

\begin{lemma}\label{lemma:sup_abs_bound}Let $x\in \mathcal{X}, f : \mathcal{X} \to \R, g : \mathcal{X}\times \mathcal{Y} \to \R$. We have the following inequality : 
    \begin{equation*}
        \left| \underset{y\in\mathcal{Y}}{\sup} \ g(x, y)  - f(x) \right| \leqslant \underset{y\in\mathcal{Y}}{\sup}\Big| {g(x, y)} - f(x) \Big|   
    \end{equation*}
\end{lemma}
\begin{proof}
    This result follows from the fact that for any $E\subseteq \R$, $|\sup{E}|\leqslant\sup{|E|}$.
\end{proof}

We now prove Lemma~\ref{prop:ineq} in Section~\ref{sec:main}.

\begin{proof}[Proof of Lemma~\ref{prop:ineq}]

The evidence lower bound gives us : 
\begin{align*}
    \log p(\bA | \bTheta, \balpha) &\geqslant \sum_{ijkl}{\tau_{ik}\tau_{jl}\log{p(A_{ij} | \Theta_{kl})}} + \sum_{ik}{\tau_{ik}\log{\frac{\alpha_k}{\tau_{ik}}}} \\
\end{align*}

Now, taking $\alpha_k = N^{-1}\sum_{i=1}^N{\tau_{ik}}$, we get : 

\begin{align*}
  \underset{\balpha}{\sup} \log p(\bA | \bTheta, \balpha) &\geqslant \log p(\bA | \bTheta, \balpha) \\
  &\geqslant \sum_{ijkl}{\tau_{ik}\tau_{jl}\log{p(A_{ij} | \Theta_{k\ell})}} + \sum_{k=1}^K\sum_{i=1}^N{\tau_{ik}\log{{\alpha}_k}- \sum_{k=1}^K\sum_{i=1}^N\tau_{ik}\log{\tau_{ik}}} \\
  & \geqslant \sum_{ijkl}{\tau_{ik}\tau_{jl}\log{p(A_{ij} | \Theta_{k\ell})}} + N \sum_{k=1}^K{{\alpha}_k\log{{\alpha}_k}} \text{ \quad as } -\sum_{ik}{\tau_{ik}\log{\tau_{ik}}} \geqslant 0 \\
  & \geqslant \sum_{ijkl}{\tau_{ik}\tau_{jl}\log{p(A_{ij} | \Theta_{k\ell})}} - N \log{K}
\end{align*}

using $\sum_{k=1}^K{{\alpha}_k\log{{\alpha}_k}} \geqslant - \log{K}$ in the last inequality.

Then, maximizing wrt $\tau$, we get : 

\begin{equation}
 \frac{1}{N^2}\underset{\balpha}{\sup} \log p(\bA | \bTheta, \balpha) \geqslant -\srGW(\bA, \bTheta) - \frac{\log{K}}{N}    
\end{equation}

On the other hand, 

\begin{align*}
     \log p(\bA | \bTheta, \balpha) &= \log \left( \sum_{z\in \mathcal{Z}_{N,K}}{ \exp \left( \sum_{i\neq j k\ell}{z_{ik}z_{jl}\log{p(A_{ij}| \Theta_{k\ell})}}\right) \P(Z=z)} \right) \\
     &\leqslant \underset{\tau \in [0, 1]^{N\times K}, \sum_{k}\tau_{ik}=1 }{\max}{\left\{\sum_{ijkl}{\tau_{ik}\tau_{jl}\log{p(A_{ij} | \Theta_{k\ell})}}\right\}}  \\
     &= - N^2 \srGW(\bA, \bTheta)
\end{align*}
  
Maximizing over $\balpha$ on the left-hand side and dividing by $N^2$ gives us 

\begin{equation}
    \frac{1}{N^2}\underset{\balpha}{\sup} \log p(\bA | \bTheta, \balpha) \leqslant - \srGW(\bA, \bTheta)
\end{equation}

We end up with 

\begin{equation}
\begin{split}
       &-\srGW(\bA, \bTheta) - \frac{\log{K}}{N}     \leqslant \frac{1}{N^2}\underset{\balpha}{\sup} \log p(A | \bTheta, \balpha) \leqslant - \srGW(\bA, \bTheta) \\
      \Longleftrightarrow
       &- \frac{\log{K}}{N} \leqslant \frac{1}{N^2}\underset{\balpha}{\sup} \log p(\bA | \bTheta, \balpha) + \srGW(\bA, \bTheta) \leqslant 0 \\
       \Longrightarrow &\left| \frac{1}{N^2}\underset{\balpha}{\sup} \log p(\bA | \bTheta, \balpha) + \srGW(\bA, \bTheta) \right| \leqslant \frac{\log{K}}{N}
\end{split}
\end{equation}

This final inequality holds whatever $\bTheta$ so that we obtain the claim. 
\end{proof}

We are now ready to prove Theorem~\ref{Th:Pconv} in Section~\ref{sec:main}.

\begin{proof}[Proof of Theorem~\ref{Th:Pconv}] 
In virtue of Theorem 3.6 of \cite{ConsistencySBM}, we know that the log-likelihood of the SBM converges uniformly to a function $\bTheta \mapsto \mathbb{M}(\bTheta)$, such that 
\begin{equation*}
    \sup_{\bTheta, \balpha} \left| \frac{1}{N^2}\log p(A | \bTheta, \balpha) - \mathbb{M}(\bTheta) \right| \underset{N\to +\infty}{\overset{\P}{\xrightarrow{\hspace{0.9cm}}}}  0.
\end{equation*}
Then we have that
\begin{equation*}
\begin{split}
    & \left| - \srGW(\bA, \bTheta) - \mathbb{M}(\bTheta) \right| \\
    &= \left| - \srGW(\bA, \bTheta) - \frac{1}{N^2}\sup_{\balpha} \log p(\bA | \bTheta, \balpha) + \frac{1}{N^2}\sup_{\balpha} \log p(\bA | \bTheta, \balpha) -  \mathbb{M}(\bTheta) \right|\\
    &\leqslant\left| \srGW(\bA, \bTheta) + \frac{1}{N^2}\sup_{\balpha} \log p(\bA | \bTheta, \balpha) \right| + \left|\frac{1}{N^2}\sup_{\balpha} \log p(\bA | \bTheta, \balpha) -  \mathbb{M}(\bTheta) \right| \\
    &\leqslant\frac{\log K}{N} + \sup_{\balpha} \left|\frac{1}{N^2}\log p(\bA | \bTheta, \balpha) -  \mathbb{M}(\bTheta) \right|
\end{split}
\end{equation*}
Where the last inequality comes from bounding the first term using Proposition \ref{Th:Pconv} and bounding the second term using Lemma \ref{lemma:sup_abs_bound}. 
This inequality is valid for any $\bTheta$ so it is also valid for the supremum in $\bTheta$ i.e.
\begin{equation} \label{eq:unif_conv_nll}
    \sup_{\bTheta}\left| - \srGW(\bA, \bTheta) - \mathbb{M}(\bTheta) \right| \leqslant\frac{\log K}{N} + \sup_{\bTheta, \balpha} \left|\frac{1}{N^2}\log p(\bA | \bTheta, \balpha) -  \mathbb{M}(\bTheta) \right|
\end{equation} 

Therefore using equation \eqref{eq:unif_conv_nll}, we have that $\bTheta \mapsto -\srGW(\bA, \bTheta)$ converges uniformly to $\mathbb{M}$.
Since we minimize $\bTheta \mapsto \srGW(\bA,\bTheta)$ in practice, we rewrite the previous result in terms of the objective function $f_N(\bTheta) := \srGW(\bA,\bTheta)$ and its limit $f(\bTheta) := -\mathbb{M}(\bTheta)$. Equation~\eqref{eq:unif_conv_nll} implies that
\[
\sup_{\bTheta} \left| f_N(\bTheta) - f(\bTheta) \right| \xrightarrow[]{\mathbb{P}} 0,
\]
i.e., $f_N$ converges uniformly to $f$.

As a consequence, we have using the classical relation \cite[Section 5.2]{van1998asymptotic}
\[
\left| \inf_{\bTheta} f_N(\bTheta) - \inf_{\bTheta} f(\bTheta) \right|
\leqslant\sup_{\bTheta} | f_N(\bTheta) - f(\bTheta) |,
\]
which shows that the minimal values also converge.

Under standard identifiability conditions ensuring that $f$ admits a unique minimizer $\bTheta^\star$, which therefore coincides with the maximizer of $\bTheta \mapsto \mathbb{M}(\bTheta)$ recovered by the negative log-likelihood, it follows that any sequence of minimizers of $\srGW(\bA,\bTheta)$ converges in probability to $\bTheta^\star$.

\end{proof}

\subsection{Efficient computation of the srGW for statistical inner loss}\label{sec:appendix_various_losses} In \cite{PeyreGW}, authors found efficient computation of Gromov-Wasserstein is guarenteed if the inner loss $\ell$ of the GW can be decomposed in the form $\ell(a, b) = f(a) + g(b) - h_1(a)h_2(b)$, also used in \cite{van2025distributional}. In our setting, the inner loss is a negative log likelihood function. It turns out that for most  classical models including zero inflated exponential likelihood, the neg log likelihood can be decomposed in this form.

\begin{itemize}
    \item Gaussian model : $\ell(a, b) = (a-b)^2 = a^2 + b^2 - 2ab$
    \item Bernoulli model : $\ell(a, b) = -\log(1-b) - a\log(b/(1-b))$
    \item Poisson model : $\ell(a, b) = b - a\log(b) + \log(a!)$ 
    \item Exponential model : $\ell(a, b) = -\log(b) + ba$
\end{itemize}

The losses of gaussian and bernoulli models already appear in GW litterature. Indeed, they correspond to the classical square loss used in the seminal paper  \cite{PeyreGW} and the binary cross entropy loss shown to be relevent in \cite{van2025distributional} to perform neighbor embedding in the context of dimensionality reduction. It is important that the srGW barycenter should not have trivial solution so that the problem is well-posed. As shown in \cite{van2025distributional}, trivial solutions of srGW exist when the inner loss $\ell$ is not a proper divergence. From a statistical point of view, we will require that the statistical model considered is indentifiable so that the KL defined by $\displaystyle \ell(\theta_1, \theta_2) = \E_{a\sim p(\cdot | \theta_1)}\left[\log \frac{p(a| \theta_1)}{p(a| \theta_2)}\right]$ is a proper divergence. 

For these use cases, we can consider a variant of the Proposition 3 in \cite{PeyreGW} to derive closed-form solutions for $\widehat{\bTheta}$. One can easily show that for a composite loss function $\ell$, such that $g$ and $h_2$ are differentiable and $(g'/h2')$ is invertible, then
\begin{equation}
\widehat{\bTheta} = (g’/ h_2’)^{-1} \left( \bT^\top h_1(\bA) \bT \oslash \bh \bh^\top \right)
\label{eq:closed_theta}
\end{equation}
where $\bh = \bT^\top \mathbf{1}_N$. This equation allows to check simply that for Gaussian, Bernoulli and Poisson models, $\widehat{\bTheta}_{1} = \bT^\top \bA \bT \oslash \bh \bh^\top 
$, and the inverse for the exponential model.

\section{Selection of \texorpdfstring{$K$}{}}\label{app:mod_sel}

As stated in Eq.~\eqref{sbm_estimators}, the estimators $\widehat{\bT}$ and $\widehat{\bTheta}$ are obtained by solving the following minimization problem
\[
f(K):=\min_{\bTheta \in \mathbb{R}^{K \times K}, \bT \in \mathcal{U}_K(\mathbf{1}_N/N)}  \underbrace{\sum_{i,j}\sum_{k,l} \ell(A_{ij}|\Theta_{kl})T_{ik}T_{jl}}_{L(\bT, \bTheta, K)}.
\]
In case the actual $K$ (say $K^*$) is unknown, one may hope that fixing a high $K >> K^*$ in the above minimization problem might lead to sparse solutions, in particular to as many empty columns of $\widehat{\bT}$ as the number of exceeding clusters. Unfortunately, this is not what happens, because of the following 
\begin{proposition}
    $f(\cdot)$ is monotonically decreasing in $K$.
\end{proposition}
\begin{proof}[sketch of proof.]
    Let us denote $\mathcal{I}_K:= \lbrace(\bTheta, \bT)| \bTheta \in \mathbb{R}^{K \times K}, \bT \in \mathcal{U}_K(\mathbf{1}_N/N) \rbrace$. Then, for any $K' > K$ it is always possible to augment $(\bTheta, \bT) \in \mathcal{I}_K$ by zero padding, in such a way to make them elements of $\mathcal{I}_{K'}$. In this sense $\mathcal{I}_K \subseteq \mathcal{I}_{K'}$ and $\min_{\mathcal{I}_{K'}} L(\bT, \bTheta, K')\leqslant\min_{\mathcal{I}_K} L(\bT, \bTheta,K)$.
\end{proof}

In practice, we observed that when solving $f(K)$, we obtain dense estimates meaning that all the available clusters of nodes are occupied (i.e. no empty columns in $\widehat{\bT}$). For instance, if the input graph is generated by SBM with $K=2$ (strong) communities and we solve $f(4)$, the resulting clustering consists in 4 communities obtained by splitting each one of the two actual communities in two. 

\section{More on the experiments}\label{app:exp}
In this section, we report additional figures and plots regarding experiments of Section 5.

\subsection{Implementation}

The implementation of the method was based on the implementations of the Python Optimal Transport library \cite{flamary2024pot} and \cite{srGW}. The only addition of the method is the computation of the closed form solution for $\widehat{\bTheta}$ when the transport is fixed.

\subsection{Connectivity matrices}
\begin{figure}[!h]
    \centering
\includegraphics[width=\linewidth]{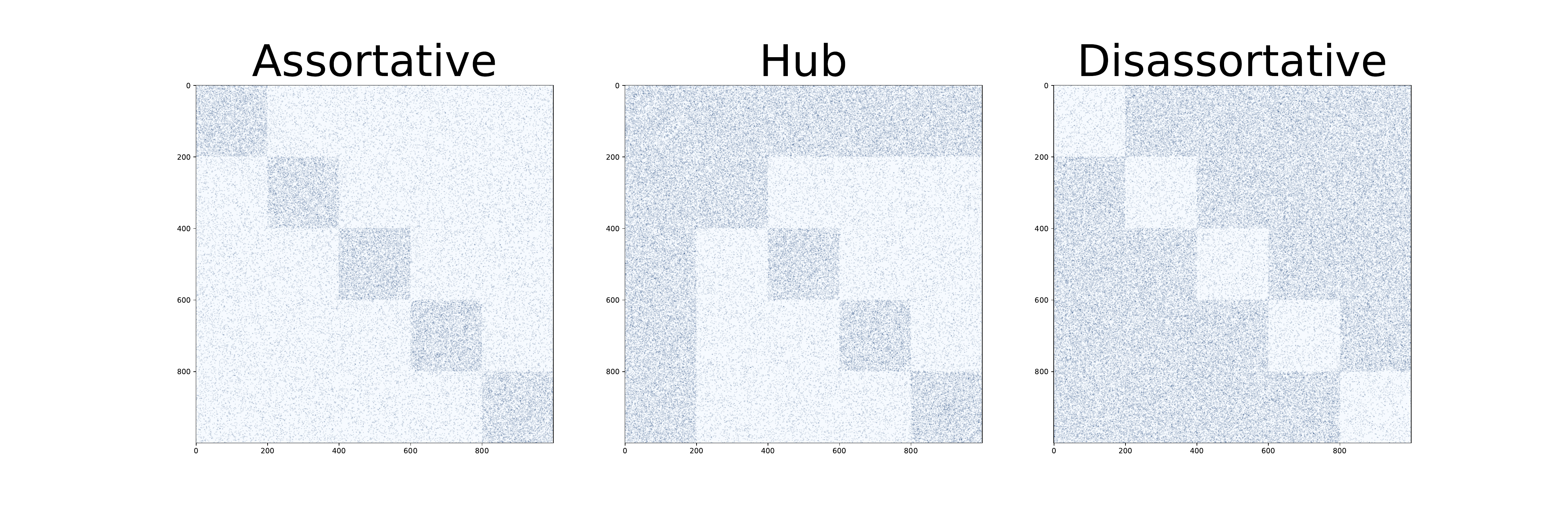}
    \caption{Sampled Adjacency matrices for the 3 SBM scenarions described in Section~\ref{sec:exp}, i.e.  \emph{Assortative}, \emph{Hub} and \emph{Disassortative.}}
    \label{fig:enter-label}
\end{figure}
\clearpage
\subsection{Unbalanced setting in SBM}

\begin{figure}[h!]
        \centering
        \makebox[\textwidth]{%
    \includegraphics[width=1\textwidth]{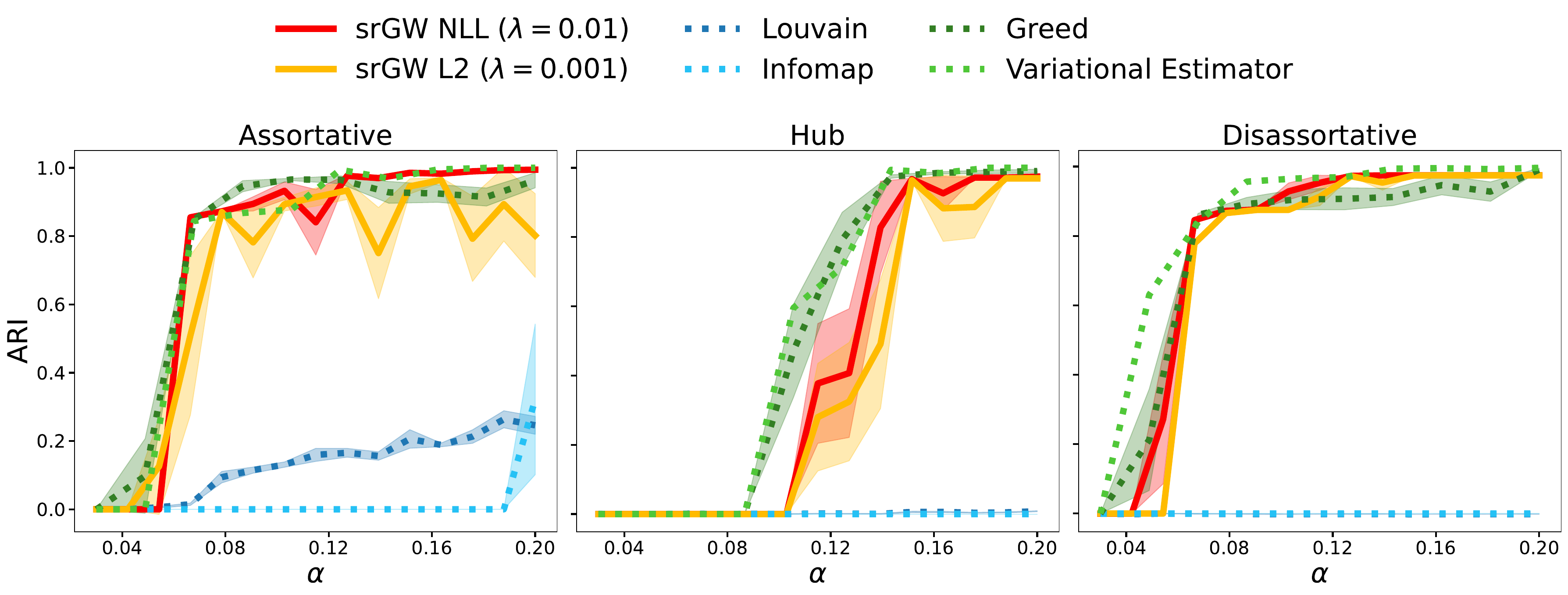}
}

\caption{Evolution of ARI wrt $\alpha$, the greater $\alpha$, the easier to detect the clusters. We take a graph with $N=10^3$ nodes. $\beta = 0.03$ and $\alpha \in [\beta, 0.2]$. Each algorithm, excepts Louvain is searching for at most $K=20$ clusters, the real number of cluster is $K^* = 5$. Here the classes are sampled according to the multinomial law $\P(k) \propto 1/k^2$.}
\end{figure}

\subsection{Runtimes}\label{runtimes}

In this section, we report the running times of each of the methods associated with the partitioning experiment in Figure \ref{fig:runtimes}. Each method was run on a CPU. As expected, Infomap and Louvain were the fastest. The method based on srGW NLL is faster than the Variational Estimator method and the Greed method. This illustrates the advantage of our approach, which bypasses the computationally demanding model selection method used by Variational Estimator.
\begin{figure}[h!]
    \centering
    \includegraphics[width=0.9\linewidth]{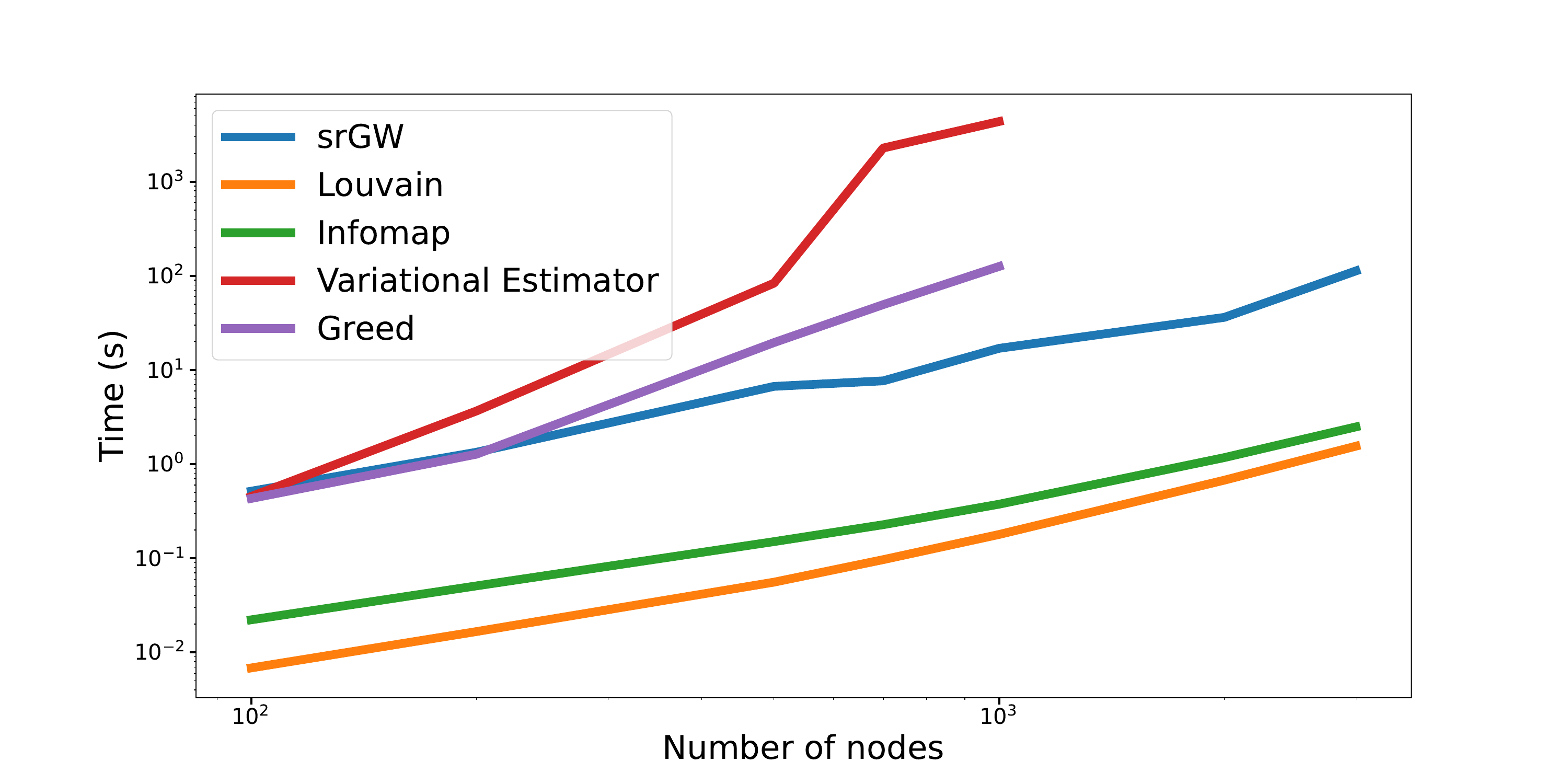}
    \caption{Running times for different graph clustering methods.}
    \label{fig:runtimes}
\end{figure}

Finally, Figure \ref{fig:cpugpu} compares running times of srGW methods on CPU and GPU, \emph{after} the Spectral Clustering initialization.
\begin{figure}[t!]
    \centering
\includegraphics[width=0.9\linewidth]{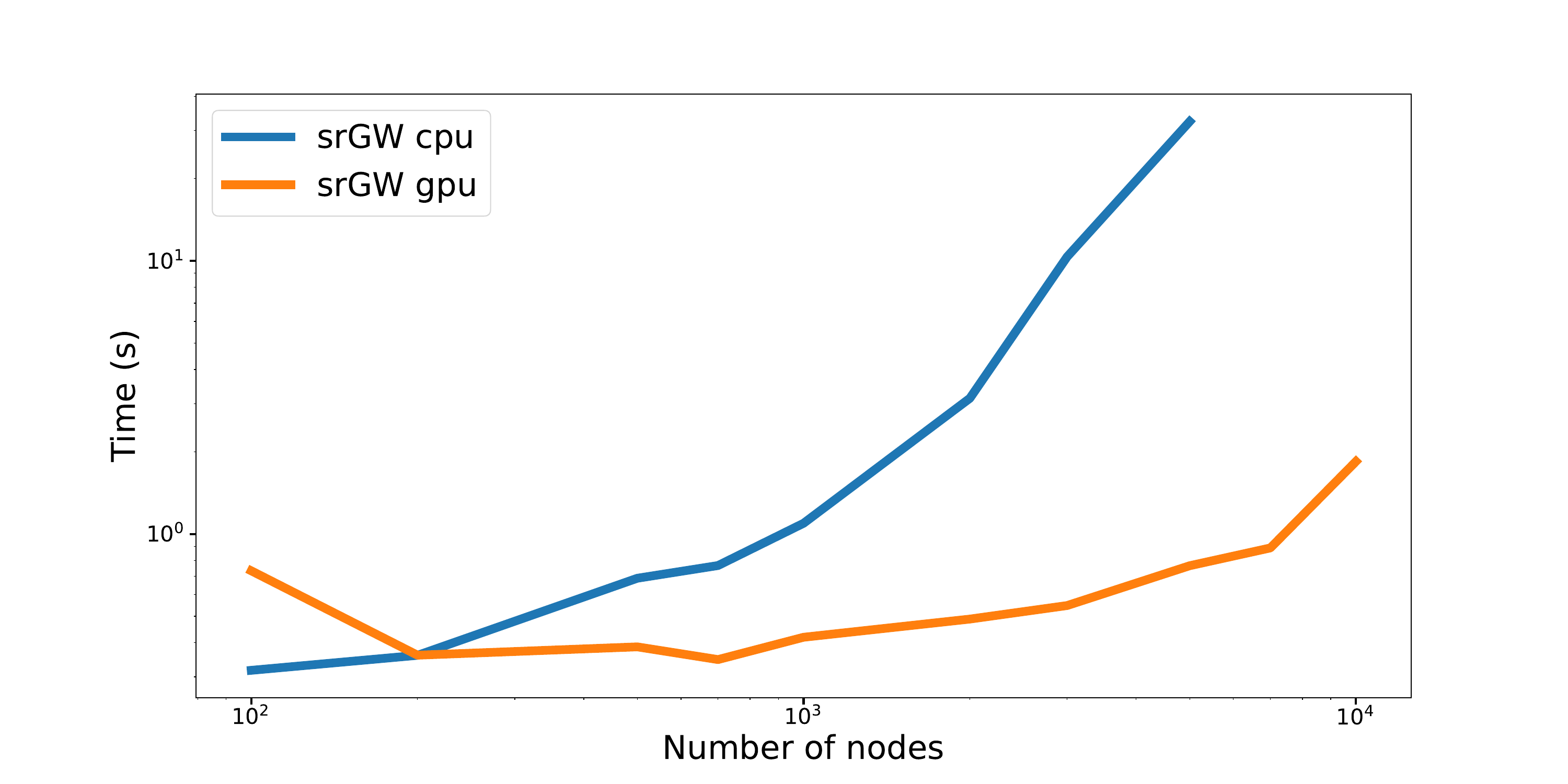}
    \caption{Running Times of srGW NLL on cpu vs gpu.}
    \label{fig:cpugpu}
\end{figure}

\clearpage
\newpage
\section*{NeurIPS Paper Checklist}

The checklist is designed to encourage best practices for responsible machine learning research, addressing issues of reproducibility, transparency, research ethics, and societal impact. Do not remove the checklist: {\bf The papers not including the checklist will be desk rejected.} The checklist should follow the references and follow the (optional) supplemental material.  The checklist does NOT count towards the page
limit. 

Please read the checklist guidelines carefully for information on how to answer these questions. For each question in the checklist:
\begin{itemize}
    \item You should answer \answerYes{}, \answerNo{}, or \answerNA{}.
    \item \answerNA{} means either that the question is Not Applicable for that particular paper or the relevant information is Not Available.
    \item Please provide a short (1–2 sentence) justification right after your answer (even for NA). 
\end{itemize}

{\bf The checklist answers are an integral part of your paper submission.} They are visible to the reviewers, area chairs, senior area chairs, and ethics reviewers. You will be asked to also include it (after eventual revisions) with the final version of your paper, and its final version will be published with the paper.

The reviewers of your paper will be asked to use the checklist as one of the factors in their evaluation. While "\answerYes{}" is generally preferable to "\answerNo{}", it is perfectly acceptable to answer "\answerNo{}" provided a proper justification is given (e.g., "error bars are not reported because it would be too computationally expensive" or "we were unable to find the license for the dataset we used"). In general, answering "\answerNo{}" or "\answerNA{}" is not grounds for rejection. While the questions are phrased in a binary way, we acknowledge that the true answer is often more nuanced, so please just use your best judgment and write a justification to elaborate. All supporting evidence can appear either in the main paper or the supplemental material, provided in appendix. If you answer \answerYes{} to a question, in the justification please point to the section(s) where related material for the question can be found.

IMPORTANT, please:
\begin{itemize}
    \item {\bf Delete this instruction block, but keep the section heading ``NeurIPS Paper Checklist"},
    \item  {\bf Keep the checklist subsection headings, questions/answers and guidelines below.}
    \item {\bf Do not modify the questions and only use the provided macros for your answers}.
\end{itemize}


\begin{enumerate}

\item {\bf Claims}
    \item[] Question: Do the main claims made in the abstract and introduction accurately reflect the paper's contributions and scope?
    \item[] Answer: \answerYes{} 
    \item[] Justification: The claims made in the abstract are supported by theoretical analysis (Sections 3 and 4) and numerical experiments (Section 5) presented in the core of the paper.

\item {\bf Limitations}
    \item[] Question: Does the paper discuss the limitations of the work performed by the authors?
    \item[] Answer: \answerYes{}, 
    \item[] Justification: The limitations of the work are discussed in the core of the paper in the section on experiments and in the conclusion.

\item {\bf Theory assumptions and proofs}
    \item[] Question: For each theoretical result, does the paper provide the full set of assumptions and a complete (and correct) proof?
    \item[] Answer:  \answerYes{}
    \item[] Justification: All the novel Propositions and Theorems are provided with the full set of assumptions and correct proofs in the appendix.

    \item {\bf Experimental result reproducibility}
    \item[] Question: Does the paper fully disclose all the information needed to reproduce the main experimental results of the paper to the extent that it affects the main claims and/or conclusions of the paper (regardless of whether the code and data are provided or not)?
    \item[] Answer: \answerYes{}
    \item[] Justification: The paper describes the experimental setup of the partitioning task for Bernoulli SBMs in three scenarios named Assortative, Hub and Disassortative. The novel estimation method is described and all the competitors are cited.

\item {\bf Open access to data and code}
    \item[] Question: Does the paper provide open access to the data and code, with sufficient instructions to faithfully reproduce the main experimental results, as described in supplemental material?
    \item[] Answer: \answerNo{}
    \item[] Justification: The code is not provided.

\item {\bf Experimental setting/details}
    \item[] Question: Does the paper specify all the training and test details (e.g., data splits, hyperparameters, how they were chosen, type of optimizer, etc.) necessary to understand the results?
    \item[] Answer: \answerYes{}
    \item[] Justification: The experimental settings and choices of hyperparameters, as well as data generation, are described in the Section 5.

\item {\bf Experiment statistical significance}
    \item[] Question: Does the paper report error bars suitably and correctly defined or other appropriate information about the statistical significance of the experiments?
    \item[] Answer: \answerYes{}
    \item[] Justification: Figures of experiment in Section 5 and in Appendix C report results as mean and standard deviation over 5 runs, indicating statistical variability.

\item {\bf Experiments compute resources}
    \item[] Question: For each experiment, does the paper provide sufficient information on the computer resources (type of compute workers, memory, time of execution) needed to reproduce the experiments?
    \item[] Answer: \answerYes{} 
    \item[] Justification: The running times of the main experiment of Section 5 is provided in Appendix, with a specific comparison of cpu and gpu running times for the method presented in the paper.
    
\item {\bf Code of ethics}
    \item[] Question: Does the research conducted in the paper conform, in every respect, with the NeurIPS Code of Ethics \url{https://neurips.cc/public/EthicsGuidelines}?
    \item[] Answer: \answerYes{}
    \item[] Justification: The research is a theoretical analysis of novel estimators of Stochastic Block Models. It doesn't involve obviously ethically sensitive applications.

\item {\bf Broader impacts}
    \item[] Question: Does the paper discuss both potential positive societal impacts and negative societal impacts of the work performed?
    \item[] Answer: \answerNo{}
    \item[] Justification: The paper focus on theoretical analysis of the stochastic block model and does not provide a discussion of broader positive or negative societal impacts of network analysis.  
    
\item {\bf Safeguards}
    \item[] Question: Does the paper describe safeguards that have been put in place for responsible release of data or models that have a high risk for misuse (e.g., pretrained language models, image generators, or scraped datasets)?
    \item[] Answer: \answerNA{}
    \item[] Justification: There is no release of data or model.

\item {\bf Licenses for existing assets}
    \item[] Question: Are the creators or original owners of assets (e.g., code, data, models), used in the paper, properly credited and are the license and terms of use explicitly mentioned and properly respected?
    \item[] Answer: \answerYes{} 
    \item[] Justification: The authors of the original paper on the semi-relaxed Gromov Wasserstein, the main topic of the paper, are cited as well as the associated Python Library Python Optimal Transport which provides srGW solvers.

\item {\bf New assets}
    \item[] Question: Are new assets introduced in the paper well documented and is the documentation provided alongside the assets?
    \item[] Answer: \answerNA{}
    \item[] Justification: No asset is introduced.

\item {\bf Crowdsourcing and research with human subjects}
    \item[] Question: For crowdsourcing experiments and research with human subjects, does the paper include the full text of instructions given to participants and screenshots, if applicable, as well as details about compensation (if any)? 
    \item[] Answer: \answerNA{}
    \item[] Justification: The paper does not involve crowdsourcing nor research with human subjects.

\item {\bf Institutional review board (IRB) approvals or equivalent for research with human subjects}
    \item[] Question: Does the paper describe potential risks incurred by study participants, whether such risks were disclosed to the subjects, and whether Institutional Review Board (IRB) approvals (or an equivalent approval/review based on the requirements of your country or institution) were obtained?
    \item[] Answer: \answerNA{}.
    \item[] Justification: The research does not involve human subjects, therefore IRB approval is not applicable.

\item {\bf Declaration of LLM usage}
    \item[] Question: Does the paper describe the usage of LLMs if it is an important, original, or non-standard component of the core methods in this research? Note that if the LLM is used only for writing, editing, or formatting purposes and does not impact the core methodology, scientific rigorousness, or originality of the research, declaration is not required.
    \item[] Answer: \answerNA{}
    \item[] Justification: The core method development of this research does not involve LLMs as any important, original, or non-standard components.

\end{enumerate}

\end{document}

%% file: math_commands.tex
\DeclareMathOperator{\balpha}{\boldsymbol{\alpha}}

\DeclareMathOperator{\btau}{\boldsymbol{\tau}}
\DeclareMathOperator{\bTheta}{\boldsymbol{\Theta}}
\DeclareMathOperator{\bA}{\boldsymbol{A}}

\DeclareMathOperator{\bT}{\boldsymbol{T}}
\DeclareMathOperator{\bI}{\boldsymbol{I}}
\DeclareMathOperator{\bP}{\boldsymbol{P}}
\DeclareMathOperator{\bh}{\boldsymbol{h}}

\DeclareMathOperator{\GW}{\text{GW}}
\DeclareMathOperator{\srGW}{\text{srGW}}

\DeclareMathOperator{\R}{\mathbb{R}}

\DeclareMathOperator{\E}{\mathbb{E}}

\renewcommand{\P}{\mathbb{P}}
\renewcommand{\E}{\mathbb{E}}

\RestyleAlgo{ruled}
\setcitestyle{numbers,square}

\newtheorem{theorem}{Theorem}
\newtheorem{definition}{Definition}
\newtheorem{lemma}{Lemma}
\newtheorem{proposition}{Proposition}

%% file: references.bib
@misc{flamary2024pot,
  author = {Flamary, R{\'e}mi and Vincent-Cuaz, C{\'e}dric and Courty, Nicolas and Gramfort, Alexandre and Kachaiev, Oleksii and Quang Tran, Huy and David, Laurène and Bonet, Cl{\'e}ment and Cassereau, Nathan and Gnassounou, Th{\'e}o and Tanguy, Eloi and Delon, Julie and Collas, Antoine and Mazelet, Sonia and Chapel, Laetitia and Kerdoncuff, Tanguy and Yu, Xizheng and Feickert, Matthew and Krzakala, Paul and Liu, Tianlin and Fernandes Montesuma, Eduardo},
  title = {POT Python Optimal Transport (version 0.9.5)},
  
  year = {2024}
}

@misc{blockmodels,
      title={Blockmodels: A R-package for estimating in Latent Block Model and Stochastic Block Model, with various probability functions, with or without covariates}, 
      author={Jean-Benoist Leger},
      year={2016},
      eprint={1602.07587},
      archivePrefix={arXiv},
      primaryClass={stat.CO}, 
}

@article{ComparingPartition,
  title     = "Comparing partitions",
  author    = "Hubert, Lawrence and Arabie, Phipps",
  journal   = "J. Classif.",
  publisher = "Springer Science and Business Media LLC",
  volume    =  2,
  number    =  1,
  pages     = "193--218",
  month     =  dec,
  year      =  1985,
  language  = "en"
}

@article{v,
  author       = {C{\'{e}}dric Vincent{-}Cuaz and
                  Titouan Vayer and
                  R{\'{e}}mi Flamary and
                  Marco Corneli and
                  Nicolas Courty},
  title        = {Online Graph Dictionary Learning},
  journal      = {CoRR},
  volume       = {abs/2102.06555},
  year         = {2021},
  eprinttype   = {arXiv},
  eprint       = {2102.06555},
  bibsource    = {dblp computer science bibliography, https://dblp.org}
}

@InProceedings{OTstructureddata,
  title = 	 {Optimal Transport for structured data with application on graphs},
  author =       {Titouan, Vayer and Courty, Nicolas and Tavenard, Romain and Laetitia, Chapel and Flamary, R{\'e}mi},
  booktitle = 	 {Proceedings of the 36th International Conference on Machine Learning},
  pages = 	 {6275--6284},
  year = 	 {2019},
  editor = 	 {Chaudhuri, Kamalika and Salakhutdinov, Ruslan},
  volume = 	 {97},
  series = 	 {Proceedings of Machine Learning Research},
  month = 	 {09--15 Jun},
  publisher =    {PMLR},
}

@article{AdjustedMetricComparison, author = {Romano, Simone and Vinh, Nguyen Xuan and Bailey, James and Verspoor, Karin}, title = {Adjusting for chance clustering comparison measures}, year = {2016}, issue_date = {January 2016}, publisher = {JMLR.org}, volume = {17}, number = {1}, issn = {1532-4435},}

@article{Greed21,
	Author = {C{\^o}me, Etienne and Jouvin, Nicolas and Latouche, Pierre and Bouveyron, Charles},
	Da = {2021/12/01},
	Doi = {10.1007/s11634-021-00440-z},
	Id = {C{\^o}me2021},
	Isbn = {1862-5355},
	Journal = {Advances in Data Analysis and Classification},
	Number = {4},
	Pages = {957--986},
	Title = {Hierarchical clustering with discrete latent variable models and the integrated classification likelihood},
	Ty = {JOUR},
	Volume = {15},
	Year = {2021}}

@book{Vershynin_2026,
author = {Vershynin, Roman},
title = {High-Dimensional Probability: An Introduction with Applications in Data Science},
edition = {2},
publisher = {Cambridge University Press},
address = {Cambridge},
series = {Cambridge Series in Statistical and Probabilistic Mathematics},
year = {2026}
}

@inproceedings{kloster2014heat,
  title={Heat kernel based community detection},
  author={Kloster, Kyle and Gleich, David F},
  booktitle={Proceedings of the 20th ACM SIGKDD international conference on Knowledge discovery and data mining},
  pages={1386--1395},
  year={2014}
}

@article{holland1983stochastic,
  title={Stochastic blockmodels: First steps},
  author={Holland, Paul W and Laskey, Kathryn Blackmond and Leinhardt, Samuel},
  journal={Social networks},
  volume={5},
  number={2},
  pages={109--137},
  year={1983},
  publisher={Elsevier}
}

@article{nowicki2001estimation,
  title={Estimation and prediction for stochastic blockstructures},
  author={Nowicki, Krzysztof and Snijders, Tom A B},
  journal={Journal of the American statistical association},
  volume={96},
  number={455},
  pages={1077--1087},
  year={2001},
  publisher={Taylor \& Francis}
}

@article{fortunato2010community,
  title={Community detection in graphs},
  author={Fortunato, Santo},
  journal={Physics reports},
  volume={486},
  number={3-5},
  pages={75--174},
  year={2010},
  publisher={Elsevier}
}

@article{borgatti2009network,
  title={Network analysis in the social sciences},
  author={Borgatti, Stephen P and Mehra, Ajay and Brass, Daniel J and Labianca, Giuseppe},
  journal={science},
  volume={323},
  number={5916},
  pages={892--895},
  year={2009},
  publisher={American Association for the Advancement of Science}
}

@article{biernacki2000assessing,
  title={Assessing a mixture model for clustering with the integrated completed likelihood},
  author={Biernacki, Christophe and Celeux, Gilles and Govaert, G{\'e}rard},
  journal={IEEE transactions on pattern analysis and machine intelligence},
  volume={22},
  number={7},
  pages={719--725},
  year={2000},
  publisher={IEEE}
}

@article{mcdaid2013improved,
  title={Improved Bayesian inference for the stochastic block model with application to large networks},
  author={McDaid, Aaron F and Murphy, Thomas Brendan and Friel, Nial and Hurley, Neil J},
  journal={Computational Statistics \& Data Analysis},
  volume={60},
  pages={12--31},
  year={2013},
  publisher={Elsevier}
}

@article{karrer2011stochastic,
  title={Stochastic blockmodels and community structure in networks},
  author={Karrer, Brian and Newman, Mark EJ},
  journal={Physical Review E—Statistical, Nonlinear, and Soft Matter Physics},
  volume={83},
  number={1},
  pages={016107},
  year={2011},
  publisher={APS}
}

@article{von2007tutorial,
  title={A tutorial on spectral clustering},
  author={Von Luxburg, Ulrike},
  journal={Statistics and computing},
  volume={17},
  number={4},
  pages={395--416},
  year={2007},
  publisher={Springer}
}

@inproceedings{shah2024neurocut,
  title={Neurocut: A neural approach for robust graph partitioning},
  author={Shah, Rishi and Jain, Krishnanshu and Manchanda, Sahil and Medya, Sourav and Ranu, Sayan},
  booktitle={Proceedings of the 30th ACM SIGKDD Conference on Knowledge Discovery and Data Mining},
  pages={2584--2595},
  year={2024}
}

@article{kawamoto2018mean,
  title={Mean-field theory of graph neural networks in graph partitioning},
  author={Kawamoto, Tatsuro and Tsubaki, Masashi and Obuchi, Tomoyuki},
  journal={Advances in neural information processing systems},
  volume={31},
  year={2018}
}

@article{Lei_2015,
   title={Consistency of spectral clustering in stochastic block models},
   volume={43},
   ISSN={0090-5364},
   DOI={10.1214/14-aos1274},
   number={1},
   journal={The Annals of Statistics},
   publisher={Institute of Mathematical Statistics},
   author={Lei, Jing and Rinaldo, Alessandro},
   year={2015},
   month=feb }

@article{lee2019review,
  title={A review of stochastic block models and extensions for graph clustering},
  author={Lee, Clement and Wilkinson, Darren J},
  journal={Applied Network Science},
  volume={4},
  number={1},
  pages={122},
  year={2019},
  publisher={Springer}
}

@misc{saade2014spectralclusteringgraphsbethe,
      title={Spectral Clustering of Graphs with the Bethe Hessian}, 
      author={Alaa Saade and Florent Krzakala and Lenka Zdeborová},
      year={2014},
      eprint={1406.1880},
      archivePrefix={arXiv},
      primaryClass={cond-mat.dis-nn}, 
}

@INPROCEEDINGS{SlicedGMM,
  author={Kolouri, Soheil and Rohde, Gustavo K. and Hoffmann, Heiko},
  booktitle={2018 IEEE/CVF Conference on Computer Vision and Pattern Recognition}, 
  title={Sliced Wasserstein Distance for Learning Gaussian Mixture Models}, 
  year={2018},
  volume={},
  number={},
  pages={3427-3436},
  doi={10.1109/CVPR.2018.00361}}

@misc{OTclusteringattributed,
      title={Optimal Transport-Based Clustering of Attributed Graphs with an Application to Road Traffic Data}, 
      author={Ioana Gavra and Ketsia Guichard-Sustowski and Loïc Le Marrec},
      year={2025},
      eprint={2512.15570},
      archivePrefix={arXiv},
      primaryClass={stat.ME}, 
}

@misc{abbe2015,
      title={Recovering communities in the general stochastic block model without knowing the parameters}, 
      author={Emmanuel Abbe and Colin Sandon},
      year={2015},
      eprint={1506.03729},
      archivePrefix={arXiv},
      primaryClass={math.PR}, 
}

@article{GAUCHERMLE,
title = {Maximum likelihood estimation of sparse networks with missing observations},
journal = {Journal of Statistical Planning and Inference},
volume = {215},
pages = {299-329},
year = {2021},
issn = {0378-3758},
doi = {https://doi.org/10.1016/j.jspi.2021.04.003},
author = {Solenne Gaucher and Olga Klopp}
}

@article{Gao_2015,
   title={Rate-optimal graphon estimation},
   volume={43},
   ISSN={0090-5364},
   DOI={10.1214/15-aos1354},
   number={6},
   journal={The Annals of Statistics},
   publisher={Institute of Mathematical Statistics},
   author={Gao, Chao and Lu, Yu and Zhou, Harrison H.},
   year={2015},
   month=dec }

@book{OTVillani,
year = {2009},
author = {Villani,Cédric},
address = {Berlin},
booktitle = {Optimal transport : old and new},
isbn = {978-3-540-71049-3},
language = {eng},
publisher = {Springer},
series = {Grundlehren der mathematischen Wissenschaften},
title = {Optimal transport  : old and new / Cédric Villani},
}

@article{GWMemoli, 
 author = {M\'{e}moli, Facundo}, 
 title = {Gromov---Wasserstein Distances and the Metric Approach to Object Matching}, 
 year = {2011}, 
 issue_date = {August 2011},
 publisher = {Springer-Verlag}, 
 address = {Berlin, Heidelberg}, 
 volume = {11}, 
 number = {4}, 
 issn = {1615-3375},
 journal = {Found. Comput. Math.},
 month = aug,  
 pages = {417–487}, numpages = {71} 
 }

@book{sturm2023space,
  title={The space of spaces: curvature bounds and gradient flows on the space of metric measure spaces},
  author={Sturm, Karl-Theodor},
  volume={290},
  number={1443},
  year={2023},
  publisher={American Mathematical Society}
}

@article{mariadassou2015convergence,
  title={Convergence of the groups posterior distribution in latent or stochastic block models},
  author={Mariadassou, Mahendra and Matias, Catherine},
  journal={Bernoulli},
  volume={21},
  number={1},
  pages={537--573},
  year={2015}
}

@article{Daudin,
  TITLE = {{A mixture model for random graphs}},
  AUTHOR = {Daudin, Jean-Jacques and Picard, Franck and Robin, Stephane},
  JOURNAL = {{Statistics and Computing}},
  PUBLISHER = {{Springer Verlag (Germany)}},
  VOLUME = {18},
  NUMBER = {2},
  PAGES = {173-183},
  YEAR = {2008},
  DOI = {10.1007/s11222-007-9046-7},
  HAL_ID = {hal-01197587},
  HAL_VERSION = {v1},
}

@article{rosvall2009map,
  title={The map equation},
  author={Rosvall, Martin and Axelsson, Daniel and Bergstrom, Carl T},
  journal={The European Physical Journal Special Topics},
  volume={178},
  number={1},
  pages={13--23},
  year={2009},
  publisher={Springer}
}

@article{Louvain,
   title={Fast unfolding of communities in large networks},
   volume={2008},
   ISSN={1742-5468},
   DOI={10.1088/1742-5468/2008/10/p10008},
   number={10},
   journal={Journal of Statistical Mechanics: Theory and Experiment},
   publisher={IOP Publishing},
   author={Blondel, Vincent D and Guillaume, Jean-Loup and Lambiotte, Renaud and Lefebvre, Etienne},
   year={2008},
   month=oct, pages={P10008} }

@inproceedings{GaucherSBM,
 author = {Gaucher, Solenne and Klopp, Olga},
 booktitle = {Advances in Neural Information Processing Systems},
 editor = {M. Ranzato and A. Beygelzimer and Y. Dauphin and P.S. Liang and J. Wortman Vaughan},
 pages = {19947--19959},
 publisher = {Curran Associates, Inc.},
 title = {Optimality of variational inference for stochasticblock model with missing links},
 volume = {34},
 year = {2021}
}

@misc{vayer2025noterelationsmixturemodels,
      title={A note on the relations between mixture models, maximum-likelihood and entropic optimal transport}, 
      author={Titouan Vayer and Etienne Lasalle},
      year={2025},
      eprint={2501.12005},
      archivePrefix={arXiv},
      primaryClass={stat.ML}, 
}

@unpublished{ConsistencySBM,
  TITLE = {{Consistency of maximum-likelihood and variational estimators in the Stochastic Block Model}},
  AUTHOR = {Celisse, Alain and Daudin, J.-J. and Pierre, Laurent},
  NOTE = {working paper or preprint},
  YEAR = {2011},
  MONTH = May,
  HAL_ID = {hal-00593644},
  HAL_VERSION = {v1},
}

@inproceedings{PeyreGW,
  TITLE = {{Gromov-Wasserstein Averaging of Kernel and Distance Matrices}},
  AUTHOR = {Peyr{\'e}, Gabriel and Cuturi, Marco and Solomon, Justin},
  BOOKTITLE = {{Proc. 33rd International Conference on Machine Learning }},
  ADDRESS = {New-York, United States},
  SERIES = {Proc. 33rd International Conference on Machine Learning },
  YEAR = {2016},
  MONTH = Jun,
  HAL_ID = {hal-01322992},
  HAL_VERSION = {v1},
}

@article{srGW,
  author       = {C{\'{e}}dric Vincent{-}Cuaz and
                  R{\'{e}}mi Flamary and
                  Marco Corneli and
                  Titouan Vayer and
                  Nicolas Courty},
  title        = {Semi-relaxed Gromov Wasserstein divergence with applications on graphs},
  journal      = {CoRR},
  volume       = {abs/2110.02753},
  year         = {2021},
  eprinttype    = {arXiv},
  eprint       = {2110.02753},
  bibsource    = {dblp computer science bibliography, https://dblp.org}
}

@misc{rigollet2018EOTML,
      title={Entropic optimal transport is maximum-likelihood deconvolution}, 
      author={Philippe Rigollet and Jonathan Weed},
      year={2018},
      eprint={1809.05572},
      archivePrefix={arXiv},
      primaryClass={math.ST}, 
}

@misc{COT,
      title={Computational Optimal Transport}, 
      author={Gabriel Peyré and Marco Cuturi},
      year={2020},
      eprint={1803.00567},
      archivePrefix={arXiv},
      primaryClass={stat.ML}, 
}

@article{chowdhury2019gromov,
  title={The Gromov--Wasserstein distance between networks and stable network invariants},
  author={Chowdhury, Samir and M{\'e}moli, Facundo},
  journal={Information and Inference: A Journal of the IMA},
  volume={8},
  number={4},
  pages={757--787},
  year={2019},
  publisher={Oxford University Press}
}

@article{xu2019scalable,
  title={Scalable Gromov-Wasserstein learning for graph partitioning and matching},
  author={Xu, Hongteng and Luo, Dixin and Carin, Lawrence},
  journal={Advances in neural information processing systems},
  volume={32},
  year={2019}
}

@inproceedings{chowdhury2021generalized,
  title={Generalized spectral clustering via Gromov-Wasserstein learning},
  author={Chowdhury, Samir and Needham, Tom},
  booktitle={International Conference on Artificial Intelligence and Statistics},
  pages={712--720},
  year={2021},
  organization={PMLR}
}

@article{van2025distributional,
  title={Distributional Reduction: Unifying Dimensionality Reduction and Clustering with Gromov-Wasserstein},
  author={Van Assel, Hugues and Vincent-Cuaz, C{\'e}dric and Courty, Nicolas and Flamary, R{\'e}mi and Frossard, Pascal and Vayer, Titouan},
  journal={Transactions on Machine Learning Research Journal},
  year={2025}
}

@inproceedings{xu2021learning,
  title={Learning graphons via structured gromov-wasserstein barycenters},
  author={Xu, Hongteng and Luo, Dixin and Carin, Lawrence and Zha, Hongyuan},
  booktitle={Proceedings of the AAAI Conference on Artificial Intelligence},
  volume={35},
  number={12},
  pages={10505--10513},
  year={2021}
}

@inproceedings{dhillon2004kernel,
  title={Kernel k-means: spectral clustering and normalized cuts},
  author={Dhillon, Inderjit S and Guan, Yuqiang and Kulis, Brian},
  booktitle={Proceedings of the tenth ACM SIGKDD international conference on Knowledge discovery and data mining},
  pages={551--556},
  year={2004}
}

@inproceedings{chen2023gromov,
  title={A gromov-wasserstein geometric view of spectrum-preserving graph coarsening},
  author={Chen, Yifan and Yao, Rentian and Yang, Yun and Chen, Jie},
  booktitle={International Conference on Machine Learning},
  pages={5257--5281},
  year={2023},
  organization={PMLR}
}

@article{pereira2025survey,
  title={A survey on optimal transport for machine learning: Theory and applications},
  author={Pereira, Luiz Manella and Amini, M Hadi},
  journal={IEEE Access},
  year={2025},
  publisher={IEEE}
}

@inproceedings{xu2019gromov,
  title={Gromov-wasserstein learning for graph matching and node embedding},
  author={Xu, Hongteng and Luo, Dixin and Zha, Hongyuan and Duke, Lawrence Carin},
  booktitle={International conference on machine learning},
  pages={6932--6941},
  year={2019},
  organization={PMLR}
}

@inproceedings{chowdhury2021quantized,
  title={Quantized gromov-wasserstein},
  author={Chowdhury, Samir and Miller, David and Needham, Tom},
  booktitle={Joint European Conference on Machine Learning and Knowledge Discovery in Databases},
  pages={811--827},
  year={2021},
  organization={Springer}
}

@inproceedings{xu2020gromov,
  title={Gromov-Wasserstein factorization models for graph clustering},
  author={Xu, Hongtengl},
  booktitle={Proceedings of the AAAI conference on artificial intelligence},
  volume={34},
  number={04},
  pages={6478--6485},
  year={2020}
}

@inproceedings{vincent2021online,
  title={Online graph dictionary learning},
  author={Vincent-Cuaz, C{\'e}dric and Vayer, Titouan and Flamary, R{\'e}mi and Corneli, Marco and Courty, Nicolas},
  booktitle={International conference on machine learning},
  pages={10564--10574},
  year={2021},
  organization={PMLR}
}

@inproceedings{liu2022robust,
  title={Robust graph dictionary learning},
  author={Liu, Weijie and Xie, Jiahao and Zhang, Chao and Yamada, Makoto and Zheng, Nenggan and Qian, Hui},
  booktitle={The Eleventh International Conference on Learning Representations},
  year={2022}
}

@inproceedings{zeng2023generative,
  title={Generative graph dictionary learning},
  author={Zeng, Zhichen and Zhu, Ruike and Xia, Yinglong and Zeng, Hanqing and Tong, Hanghang},
  booktitle={International Conference on Machine Learning},
  pages={40749--40769},
  year={2023},
  organization={PMLR}
}

@article{vincent2022template,
  title={Template based graph neural network with optimal transport distances},
  author={Vincent-Cuaz, C{\'e}dric and Flamary, R{\'e}mi and Corneli, Marco and Vayer, Titouan and Courty, Nicolas},
  journal={Advances in Neural Information Processing Systems},
  volume={35},
  pages={11800--11814},
  year={2022}
}

@inproceedings{chu2023wasserstein,
  title={Wasserstein barycenter matching for graph size generalization of message passing neural networks},
  author={Chu, Xu and Jin, Yujie and Wang, Xin and Zhang, Shanghang and Wang, Yasha and Zhu, Wenwu and Mei, Hong},
  booktitle={International Conference on Machine Learning},
  pages={6158--6184},
  year={2023},
  organization={PMLR}
}

@inproceedings{qian2024reimagining,
  title={Reimagining graph classification from a prototype view with optimal transport: Algorithm and theorem},
  author={Qian, Chen and Tang, Huayi and Liang, Hong and Liu, Yong},
  booktitle={Proceedings of the 30th ACM SIGKDD conference on knowledge discovery and data mining},
  pages={2444--2454},
  year={2024}
}

@inproceedings{
krzakala2026the,
title={The quest for the {GRA}ph Level autoEncoder ({GRALE})},
author={Paul Krzakala and Gabriel Melo and Charlotte Laclau and Florence d'Alch{\'e}-Buc and R{\'e}mi Flamary},
booktitle={The Thirty-ninth Annual Conference on Neural Information Processing Systems},
year={2026},
}

@inproceedings{clark2025generalized,
  title={Generalized dimension reduction using semi-relaxed Gromov-Wasserstein distance},
  author={Clark, Ranthony A and Needham, Tom and Weighill, Thomas},
  booktitle={Proceedings of the AAAI Conference on Artificial Intelligence},
  volume={39},
  number={15},
  pages={16082--16090},
  year={2025}
}

@inproceedings{van2024graph,
  title={A graph matching approach to balanced data sub-sampling for self-supervised learning},
  author={Van Assel, Hugues and Balestriero, Randall},
  booktitle={NeurIPS 2024 Workshop: Self-Supervised Learning-Theory and Practice},
  year={2024}
}

@article{benamou2015iterative,
  title={Iterative Bregman projections for regularized transportation problems},
  author={Benamou, Jean-David and Carlier, Guillaume and Cuturi, Marco and Nenna, Luca and Peyr{\'e}, Gabriel},
  journal={SIAM Journal on Scientific Computing},
  volume={37},
  number={2},
  pages={A1111--A1138},
  year={2015},
  publisher={SIAM}
}

@article{van1998asymptotic,
  title={Asymptotic Statistics (1998)},
  author={van der Vaart, AW},
  year={1998}
}
